\documentclass[twoside]{article}
\usepackage[accepted]{aistats2014}
\usepackage{graphicx} 
\usepackage[tight]{subfigure}
\usepackage[all=normal,floats,wordspacing,paragraphs]{savetrees}
\usepackage[round]{natbib}
\usepackage{bm}
\usepackage{todonotes}
\usepackage{enumitem}
\usepackage[T1]{fontenc}  

\usepackage{algorithm}
\usepackage{algorithmicx}
\usepackage{algpseudocode}

\algnotext{EndFor}
\algnotext{EndIf}
\algnotext{EndProcedure}

\usepackage{hyperref}
\hypersetup{letterpaper,bookmarksopen,bookmarksnumbered,
pdfpagemode=UseOutlines,
colorlinks=true,
linkcolor=blue,
pdfborder={0 0 0},
anchorcolor=blue,
citecolor=blue,
filecolor=blue,
menucolor=blue,
urlcolor=blue,
pdftitle={Near Optimal Bayesian Active Learning for Decision Making},
pdfauthor={Shervin Javdani, Yuxin Chen, Amin Karbasi, Andreas Krause, J. Andrew Bagnell, Siddhartha Srinivasa}
}

\usepackage{booktabs}

\usepackage{xspace}
\usepackage{url}
\usepackage{amsmath,amssymb,amsthm,mathtools}
\usepackage{bbm}

\newcommand{\dist}{P}

\newcommand{\alg}{\textsc{HEC}\xspace}

\newcommand{\hypoth}{h}
\newcommand{\hypothset}{H}
\newcommand{\allhypoth}{\mathcal{H}}

\newcommand{\regnoarg}{r}
\newcommand{\reg}[1]{\regnoarg_{#1}}
\newcommand{\regset}{R}
\newcommand{\regsetiff}{\regset_{\text{iff}}}
\newcommand{\regsetas}{\regset_{\text{as}}}
\newcommand{\allreg}{\mathcal{R}}
\newcommand{\numreg}{\vert \allreg \vert}
\newcommand{\subregnoarg}{g}
\newcommand{\subreg}[1]{\subregnoarg_{#1}} 
\newcommand{\subregset}{G}
\newcommand{\subregsettwo}{\boldsymbol{\zeta}}
\newcommand{\allsubreg}{\mathcal{G}}
\newcommand{\subreghat}[1]{\widehat{\subregnoarg}_{#1}} 

\newcommand{\hypergraph}{\mathbf{G}}
\newcommand{\hypergraphreg}{\hypergraph^{r}}
\newcommand{\hypergraphsplit}{\hypergraph^{s}}
\newcommand{\hypernodegeneral}{X}
\newcommand{\hyperedgegeneral}{E}

\newcommand{\hyperedgeset}{\mathcal{E}}

\newcommand{\hyperedge}{e}
\newcommand{\hyperedgesetobs}{\hyperedgeset(\subtestobs)}

\newcommand{\edgeweightfunc}{w}

\newcommand{\hyperedgeelements}{\{\subreg{1},\dots,\subreg{{\maxcard}}\}}
\newcommand{\hyperedgeelementscardvar}{\{\subreg{1},\dots,\subreg{{\cardvar}}\}}

\newcommand{\marginalfHEC}{\Delta}
\newcommand{\falg}{f_{\text{\alg}}}
\newcommand{\fhec}{f_{\text{\alg}}}

\newcommand{\subtestobs}{\mathcal{S}}
\newcommand{\subtestobsunion}{\subtestobs \cup \{(\actionitem,\hypoth(\actionitem))\}}
\newcommand{\subtestobsactions}{\subtestobs_{\actionset}}

\newcommand{\settestobs}{\testset\times\observationset}

\newcommand{\actionset}{T}
\newcommand{\actionitem}{t}

\newcommand{\testset}{\mathcal{T}}

\newcommand{\observationset}{\mathcal{O}}

\newcommand{\observationitem}{o}

\newcommand{\cost}{\mathcal{C}}

\newcommand{\allgroups}{\mathcal{G}}
\newcommand{\policy}{\pi}
\newcommand{\policyHEC}{\pi_{\text{HEC}}}
\newcommand{\version}{\mathcal{V}}
\newcommand{\edgeversion}{\mathcal{E}}
\newcommand{\probacro}{DRD\xspace}

\newcommand\numberthis{\addtocounter{equation}{1}\tag{\theequation}}
\newcommand{\eref}[1]{(\eqref{#1})}
\newcommand{\sref}[1]{Sec.~\ref{#1}}
\newcommand{\figref}[1]{Fig.~\ref{#1}}
\newcommand{\algoref}[1]{Alg.~\ref{#1}}
\newcommand{\tabref}[1]{Table~\ref{#1}}
\newcommand{\thmref}[1]{Theorem~\ref{#1}}
\newcommand{\lemmaref}[1]{Lemma~\ref{#1}}
\newcommand{\propref}[1]{Proposition~\ref{#1}}

\newcommand{\maxcard}{k}
\newcommand{\maxcardiff}{\maxcard_{\text{iff}}}
\newcommand{\maxcardas}{\maxcard_{\text{as}}}
\newcommand{\cardvar}{ {\widehat{\maxcard} } }

\newcommand{\bigo}[1]{\mathcal{O}(#1)}
\newcommand{\indicator}[1]{\mathbbm{1}(#1)}

\newcommand{\powersum}{PS}
\newcommand{\comsympoly}{CHP}

\DeclareMathOperator*{\argmax}{arg\,max}

\newtheorem{theorem}{Theorem}
\newtheorem{lemma}{Lemma}
\newtheorem{sublemma}{Lemma}[lemma]
\newtheorem{subsublemma}{Lemma}[sublemma]
\newtheorem{proposition}{Proposition}

\newcommand*\circled[1]{\tikz[baseline=(char.base)]{\node[shape=circle,draw,inner sep=2pt] (char) {#1};}}

\begin{document}
\setlength{\textfloatsep}{10.0pt plus 2.0pt minus 4.0pt}
\setlength{\dbltextfloatsep}{10.0pt plus 2.0pt minus 4.0pt}




%

%

\twocolumn[

\aistatstitle{Near Optimal Bayesian Active Learning for Decision Making}
\aistatsauthor{Shervin Javdani \And Yuxin Chen \And Amin Karbasi}
\aistatsaddress{Carnegie Mellon University \And ETH Z{\"u}rich \And ETH Z{\"u}rich}
\aistatsauthor{Andreas Krause \And J. Andrew Bagnell \And Siddhartha Srinivasa}
\aistatsaddress{ETH Z{\"u}rich \And Carnegie Mellon University \And Carnegie Mellon University} ]

\runningauthor{Shervin Javdani, Yuxin Chen, Amin Karbasi, Andreas Krause, J. Andrew Bagnell, Siddhartha Srinivasa }

\begin{abstract}

How should we gather information to make effective decisions? We address Bayesian active learning and experimental design problems, where we sequentially select tests to reduce uncertainty about a set of hypotheses. Instead of minimizing uncertainty per se, we consider a set of overlapping \emph{decision regions} of these hypotheses. Our goal is to drive uncertainty into a single decision region as quickly as possible. 

We identify necessary and sufficient conditions for correctly identifying a decision region that contains all hypotheses consistent with observations. We develop a novel {\em Hyperedge Cutting} (\alg) algorithm for this problem, and prove that is competitive with the intractable optimal policy. Our efficient implementation of the algorithm relies on computing subsets of the complete homogeneous symmetric polynomials. Finally, we demonstrate its effectiveness on two practical applications: approximate comparison-based learning and active localization using a robot manipulator.

\end{abstract}


\section{Introduction}\label{sec:intro}

Bayesian active learning addresses the problem of selecting a sequence of experiments, or \emph{tests}, to determine a hypothesis consistent with observations. This fundamental problem arises in a wide range of applications such as medical procedures, content search, and robotics. It has been studied in several domains, including machine learning~\citep{dasgupta04,balcan06agnostic,nowak_noisy_gbs}, statistics~\citep{lindley_1956,chaloner_1995}, decision theory~\citep{howard66voi}, and others.

For instance, in {\em automated medical diagnosis}~\citep{kononenko2001} we are presented with hypotheses about the state of a patient, and select medical tests to infer their illness. In {\em comparison-based learning}~\citep{goyal_2008, karbasi2012},  we infer a target in a database by sequentially presenting a user with pairs of candidates, and having the user select which is closer. In robotic \emph{active localization}, the robot attempts to identify its own or an object's location by probing, e.g., with touch or vision~\citep{fox_1998_active_localization, kollar2008efficient, hsiao_2008_robust_belief, javdani_2013_touchloc}. In general, the goal is to gather the necessary information while minimizing test cost.

In this paper, we develop a general framework for addressing these problems. Instead of indiscriminately minimizing uncertainty about hypotheses directly, we aim to reduce uncertainty in a structured way to facilitate \emph{decision making}. We suppose the hypothesis space is covered by a set of decision regions: Each region identifies the set of hypotheses for which it would succeed. Our goal is to select tests that quickly concentrate all consistent hypotheses in a single decision region.

Special cases of this general problem have been studied. In the so called {\em Optimal Decision Tree} (ODT) problem, each decision region corresponds to a single hypothesis. In this case, a greedy algorithm called \emph{Generalized Binary Search} (GBS) is known to perform near optimally, i.e., the expected number of observations is $O(\log n)$ more than the optimum policy where $n$ indicate the number of hypotheses~\citep{dasgupta04, guillory09, kosaraju99}. GBS greedily selects tests in expectation over the test outcomes to maximize the probability mass of eliminated hypotheses. Another special instance of our setting is the \emph{Equivalence Class Determination} (ECD) problem~\citep{golovin_bayesian_noisy_obs} where the set of hypotheses is (disjointly) partitioned-- that is, decision regions do not overlap and collectively cover the set of hypotheses. In this case, it is known that GBS performs poorly while greedily optimizing a more informative objective known as EC2 exhibits an $O(\log n)$ approximation guarantee~\citep{golovin_bayesian_noisy_obs}.   

In both aforementioned settings, decision regions are \emph{disjoint}. In this paper, we tackle the general case of \emph{overlapping decision regions}, a problem that is less understood. We develop a novel surrogate objective function, which we call {\emph{Hyperedge Cutting} (\alg)}, and prove that the policy which greedily maximizes this objective has strong theoretical guarantees. It relies on the fact that our proposed objective function satisfies \emph{adaptive submodularity}~\citep{golovin_adaptive_2011}, a natural diminishing returns property that generalizes the classical notion of submodularity to policies. 

We empirically evaluate our algorithm on two applications: approximate comparison-based learning~\citep{karbasi2012}, and active localization with a robot hand. In approximate comparison-based learning, a user is searching through set of items (e.g., movies), and is not particularly interested in a single item, but rather any suggestion from a given category (e.g., the horror genre). The search terminates once all items consistent with user responses are contained in a single category. Similarly, many actions in robotic manipulation, such as pushing a button or grasping an object, inherently tolerate some uncertainty. The robot need not know the exact location of an object, but rather must localize an object to a decision region to ensure it can successfully accomplish the task. An optimal policy achieves each of these with the smallest test cost.


We make the following contributions:
\begin{enumerate}[topsep=0pt,itemsep=0ex,partopsep=0ex,parsep=1ex]
\item We provide a necessary and sufficient condition for identifying if a decision region contains all hypotheses that are consistent with the tests performed.
\item We develop a novel algorithm -- {\emph{Hyperedge Cutting} (\alg)} -- and prove that it is competitive with the intractable optimal algorithm.
\item We provide an efficient way to implement our algorithm based on computing sums of the complete homogeneous symmetric polynomials.
\item We demonstrate the empirical effectiveness of our approach for both comparison-based learning and active localization in a manipulation task.
\end{enumerate}

\section{Problem Statement}\label{sec:problem}

\begin{figure*}[t!h]
\centering
 \subfigure[\emph{Regions and hypotheses}]{
\includegraphics[width=.31\textwidth]{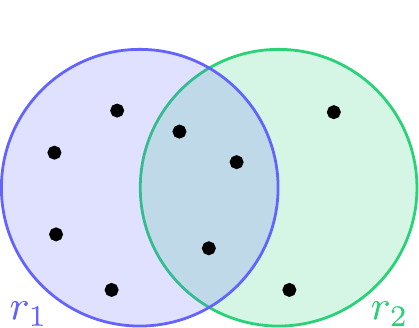}
 \label{fig:hypregion}
 }\enskip
\subfigure[\emph{Subregions and hypergaph}]{
\includegraphics[width=.31\textwidth]{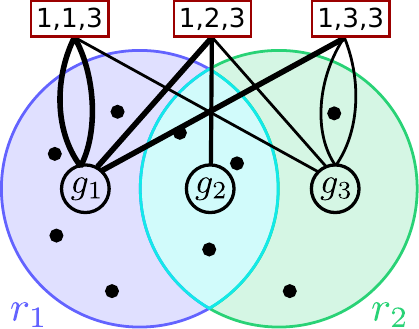}
 \label{fig:hypergraph}
 } \enskip
 \subfigure[\emph{Edges cut if all $\hypoth \in \subreg{3}$ inconsistent}]{
\includegraphics[width=.31\textwidth]{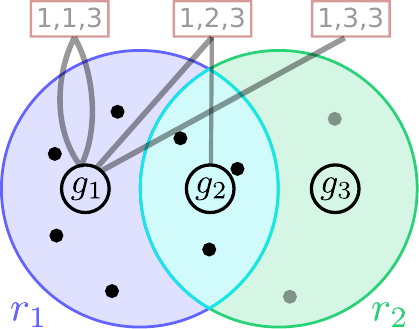}
 \label{fig:hypergraphcut}
 }\vspace{-2mm}
 \caption{\footnotesize \subref{fig:hypregion} An instance of the Decision Region Determination (\probacro) problem with two decision regions. Black dots represent hypotheses and circles represent decision regions.~\subref{fig:hypergraph} The resulting subregions and splitting hyperedges constructed by Hyperedge Cutting (\alg) algorithm. Thickness of edge represents weight, which is proportional to weight in subregion.~\subref{fig:hypergraphcut} Resulting hypergraph when all hypotheses in subregion $\subreg{3}$ inconsistent, causing all edges to be ``cut''.
 } \label{fig:hypeg} 
\end{figure*} 

We formalize our Bayesian active learning problem by assuming a prior probability distribution $\dist$ on a set of hypotheses $\allhypoth$ (e.g., state of patient, location of target). By conducting tests from a set of tests $\testset$, we gain information about the true, initially unknown hypothesis. 


More formally, for a given hypothesis $\hypoth\in\allhypoth$, running a test $\actionitem\in\testset$ produces an outcome (deterministically) from a finite set of outcomes/observations $\observationset$. Thus, each hypothesis $\hypoth\in\allhypoth$ can be considered a function $\hypoth:\testset\rightarrow\observationset$ mapping tests to outcomes.
Suppose we have executed a set of tests $\actionset=\{\actionitem_1,\dots,\actionitem_m\}\subseteq\testset$ (e.g., medical tests we ran, items shown to the user, moves made by the robot), and have observed their outcomes $\hypoth(\actionitem_1),\dots,\hypoth(\actionitem_m)$. Our evidence so far is captured by a set of test-outcome pairs, $\subtestobs \subseteq \settestobs$, where $\subtestobs=\{(\actionitem_1,\hypoth(\actionitem_1)),\dots,(\actionitem_m,\hypoth(\actionitem_m))\}$.

Upon observing $\subtestobs$, we can rule out hypotheses inconsistent with our observations. We denote the resulting set of hypotheses by
\begin{equation}
    \displaystyle\version(\subtestobs)=\{\hypoth\in\allhypoth : \forall (\actionitem,\observationitem)\in\subtestobs,\hypoth(\actionitem)=\observationitem\}\label{eq:vs}
\end{equation}
In principle, we can now choose tests that reduce our uncertainty about the set of hypotheses directly. In many practical problems, we are primarily concerned about {\em reducing uncertainty for the purpose of making a decision}: it is not necessary to remove all uncertainty, but it is necessary to reduce uncertainty in a structured way to ensure a decision action will be successful. Choosing tests that reduce uncertainty dramatically, but still leave it unclear what action to choose, will not be effective. We now formalize this idea.

\paragraph{Active learning for decision making.}
Suppose we have a set of decisions $\allreg$, and the eventual goal of selecting a decision $\regnoarg \in \allreg$ after gathering information. For example, in medical diagnosis, we choose a treatment; in robotic manipulation, we press a button (\figref{fig:rob_expr_setup}); in content search, we recommend a particular movie.

Each decision region $\regnoarg$ corresponds to the set of hypotheses for which it would succeed, i.e., $\regnoarg \subseteq \allhypoth$.
Our problem is then captured by a \emph{hypergraph}, a generalization of a graph in which an edge can connect to any number of nodes. Briefly, a hypergraph $\hypergraph$ is a pair $\hypergraph = (\hypernodegeneral,\hyperedgegeneral)$, where $\hypernodegeneral$ is a set of elements called \emph{nodes}, and $\hyperedgegeneral$ is a collection of sets of $\hypernodegeneral$ called \emph{hyperedges}. We can specify our problem with a hypergraph, which we refer to as the \emph{region hypergraph} $\hypergraphreg  = (\allhypoth, \allreg)$.\footnote{We illustrate decision regions as circles (e.g., \figref{fig:hypregion}) - however, our method treats regions as arbitrary sets.}

Note that in general, multiple decisions are equally suitable for a hypothesis: In the robot example, multiple manipulation actions may succeed for an object location (\figref{fig:rob_expr_setup}); in movie recommendation, the user may be indifferent among sets of movies. Hence, we allow the decision regions to overlap (\figref{fig:hypregion}). Formally, we also assume that the set of hypotheses is covered by the collection of decision regions, i.e., $\allhypoth = \cup_{\allreg} \regnoarg$. 

The ultimate goal is to find a policy $\policy$ for running tests that allows us to determine a decision region $\regnoarg$ the true hypothesis is guaranteed to lie in. In other words, upon termination we require that $\version(\subtestobs)\subseteq \regnoarg$ for some $\regnoarg \in \allreg$. 

Thus, we seek a policy for selecting a minimal number of tests to determine a suitable decision. A {\em policy} $\policy$ is a function from a set of evidence so far $\subtestobs$, to the next test to choose (or to stop running tests). A policy is feasible if and only if it drives all remaining uncertainty into any single decision region, $\version(\subtestobs)\subseteq \regnoarg$. We define the expected cost (i.e., number of tests\footnote{Note that while we focus on tests with unit cost, our results generalize to tests with non-uniform costs.}) of policy $\policy$ as:
$$\cost(\policy) = \sum_{\hypoth\in\allhypoth}P(\hypoth) \vert\testset(\policy,\hypoth)\vert,$$ where $\testset(\policy,h)$ is the set of tests policy $\policy$ chooses in case the correct hypothesis is $\hypoth$. 
Given this, we seek a feasible policy of minimal cost, i.e., 
\begin{equation}
\policy^*=\arg\min_{\policy} C(\policy)\text{ s.t. }\forall\hypoth, \exists \regnoarg: \version(\testset(\policy,\hypoth))\subseteq \regnoarg\label{eq:probstat}
\end{equation}
We call Problem~\eqref{eq:probstat} the {\em Decision Region Determination (\probacro) Problem}.

Special cases of Problem~\eqref{eq:probstat} have been studied before. In particular, the special case where each hypothesis is contained in a dedicated region is called the Optimal Decision Tree (ODT) problem~\citep{kosaraju99}. More generally, the special case where the regions {\em partition} the hypothesis space (i.e., do not overlap), is called the Equivalence Class Determination (ECD) Problem~\citep{golovin_bayesian_noisy_obs}. For both of these special cases, it is known that finding a policy $\policy$ for which $\cost(\policy)\leq \cost(\policy^*) o(\log n)$ is NP-hard \citep{chakaravarthy07decision}. Here, $\policy^*$ indicates the optimum policy. To the best of our knowledge, there are no efficient algorithms with theoretical approximation guarantees for the general \probacro problem. In the following, we present such an algorithm.

\section{The \alg Algorithm}\label{sec:results}
We now introduce and analyze our algorithm -- the {\em Hyperedge Cutting} (\alg) approach. 

\subsection{Overview}
Our key strategy is to transform the \probacro Problem~\eqref{eq:probstat} into an alternative representation -- a different hypergraph for splitting decision regions. Observing certain test outcomes corresponds to downweighting or cutting hyperedges in this hypergraph. The construction is chosen so that cutting all hyperedges is a necessary and sufficient condition for driving all uncertainty into a single decision region. We then prove that a simple greedy algorithm, which chooses tests that reduce hyperedge weight maximally (in expectation), implements a policy that is competitive with the optimal (intractable) policy for Problem~\eqref{eq:probstat}. In \sref{sec:implementation}, we show how this greedy algorithm can be efficiently implemented.


\subsection{Splitting hypergraph construction} 
We construct a different hypergraph, the \emph{splitting hypergraph} $\hypergraphsplit$, and define our objective on that. Here, our hyperedges are not sets, but \emph{multisets}, a generalization of sets where members are allowed to appear more than once. As a result, a node can potentially appear in a hyperedge multiple times. The cardinality of a hyperedge refers to how many nodes it is connected to. 

We observe that for solving the \probacro problem, we can group together all hypotheses that share the same region assignments. We refer to this grouping as a \emph{subregion} $\subregnoarg$, and the set of all subregions as $\allsubreg$. More formally, for any pair $\hypoth_k\in\subreg{i}$ and $\hypoth_l\in\subreg{i}$, we have $\hypoth_k\in \reg{j}$ if and only if $\hypoth_l\in \reg{j}$. In a slight abuse of notation, we say that a subregion is contained in a region, $\subregnoarg \in \regnoarg$, if  $\forall \hypoth \in \subregnoarg, \hypoth \in \regnoarg$ (\figref{fig:hypergraph}). Similarly, we say that $\hypoth \in \hyperedge$ if $\exists \subregnoarg \in \hyperedge \text{ s.t. } \hypoth \in \subregnoarg$.
It is easy to see that all remaining hypotheses $\version(\subtestobs)$ are contained in $\regnoarg$ if and only if all remaining subregions are contained in $\regnoarg$.

We construct the splitting hypergraph $\hypergraphsplit$ over these subregions. Each subregion $\subregnoarg \in \allsubreg$ corresponds to a node. The hyperedges $\hyperedge \in \hyperedgeset$ consist of all multisets of precisely $\maxcard$ subregions, $\hyperedge=\hyperedgeelements$, such that a single decision region does not contain them all (we will describe how $\maxcard$ is selected momentarily). Note that hyperedges can contain the same subregion multiple times. Formally,
\begin{equation}\label{eq:hyper}
\hyperedgeset = \{\hyperedge :  \vert\hyperedge\vert=\maxcard \wedge\nexists\; \regnoarg \text{ s.t. } \forall \hypoth \in\hyperedge, \hypoth \in  \regnoarg \}.
\end{equation}
 Our splitting hypergraph is defined as $\hypergraphsplit = (\allsubreg, \hyperedgeset)$. \figref{fig:hypergraph} illustrates the splitting hypergraph obtained from the DRD instance of \figref{fig:hypregion}.

 \paragraph{Hyperedge Cardinality $\mathbf{\maxcard}$.}
Key to attaining our results is the proper selection of hyperedge cardinality $\maxcard$. If $\maxcard$ is too small, our results won't hold, and our algorithm won't solve the \probacro problem. If $\maxcard$ is too large, we waste computational effort, and our theoretical bounds loosen. Here, we define the cardinality we use practically. Our theorems hold for a smaller, more difficult to compute $\maxcard$ as well. See Appendix for details.
\begin{equation}
    \maxcard = \min \left(\max_{\hypoth\in\allhypoth}\vert\{\regnoarg \!:\! \hypoth\in\regnoarg\}\vert, \max_{\regnoarg \in \allreg} \vert\{\subregnoarg \!:\! \subregnoarg \in \regnoarg\}\vert\right) + 1
\end{equation}
Note that each term is a property of the original region hypergraph $\hypergraphreg$ defined in \sref{sec:problem}: $\max_{\hypoth}\vert\{\regnoarg \! : \! \hypoth\in\regnoarg\}\vert$ is the maximum degree of any node, and $\max_{\regnoarg} \vert\{ \subregnoarg \! : \! \subregnoarg \in \regnoarg\} \vert$ bounds the maximum cardinality of hyperedges in $\hypergraphreg$.\footnote{It is precisely the maximum cardinality of any hyperedge if we grouped hypotheses into subregions in $\hypergraphreg$.}

\subsection{Relating \probacro and \alg}
How does the hypergraph capture our progress towards solving Problem~\eqref{eq:probstat}? Observing a set of test-outcomes $\subtestobs \subseteq \settestobs$ eliminates inconsistent hypotheses, and consequently downweights or eliminates (``cuts'') incident hyperedges (\figref{fig:hypergraphcut}). Analogous to the definition of $\version(\subtestobs)$ in \eqref{eq:edge_cons}, we define the set of hyperedges consistent with $\subtestobs$ by
\begin{equation}\edgeversion(\subtestobs)=\{\hyperedge\in\hyperedgeset : \forall (i,o)\in\subtestobs\;\forall \hypoth\in\hyperedge, \hypoth(i)=o\}\label{eq:edge_cons}\end{equation}
The following result guarantees that cutting all hyperedges is a necessary and sufficient condition for success, i.e., driving all uncertainty into a single decision region.

\begin{theorem}
  \label{theory:hyperedge_iff}  
  Suppose we construct a splitting hypergraph by drawing hyperedges of cardinality $\maxcard$ according to \eqref{eq:hyper}.  Let  $\subtestobs\subseteq\settestobs$ be a set of evidence. All consistent hypotheses lie in some decision region if and only if all hyperedges are cut, i.e.,
  $$\edgeversion(\subtestobs)=\emptyset\;\;\Leftrightarrow\;\;\exists \regnoarg:\version(\subtestobs)\subseteq\regnoarg$$
\end{theorem}
Thus, the \probacro Problem~\eqref{eq:probstat} is equivalent to finding a policy of minimum cost that cuts all hyperedges.
This insight suggests a natural algorithm: select tests that cut as many edges as possible (in expectation). In the following, we formalize this approach.

\subsection{The Hyperedge Cutting (\alg) Algorithm}
Given the above construction, we define a suitable objective function whose maximization will ensure that we pick tests to remove hyperedges quickly, thus providing us with an algorithm that identifies a correct decision region. First, we define the weight of a subregion as the sum of hypothesis weights, $p(\subregnoarg) = \sum_{\hypoth \in \subregnoarg} p(\hypoth)$. We define the weight of a hyperedge $\hyperedge=\hyperedgeelements$ as $\edgeweightfunc(\hyperedge) = \prod_{i=1}^{\maxcard} \dist({\subreg{i}})$. More generally, we define the weight of a collection of hyperedges as $\edgeweightfunc(\{\hyperedge_1,\dots,\hyperedge_n\})= \sum_{l=1}^n\edgeweightfunc(\hyperedge_l)$. Now, given a pair of test/observation $(\actionitem,\observationitem)$, we can identify the set of inconsistent hypotheses
, which in turn implies the set of hyperedges that should be downweighted or removed. Formally, given a set of test/observation pairs $\subtestobs \subseteq \settestobs$, 
we define its utility $\fhec(\subtestobs)$ as
\begin{equation}\label{eq:HEC}
\fhec(\subtestobs) = \edgeweightfunc(\hyperedgeset)-\edgeweightfunc(\edgeversion(\subtestobs)).
\end{equation}
Thus $\fhec(\subtestobs)$ is the total mass of all the edges cut via observing set $\subtestobs$.

A natural approach to the \probacro Problem is thus to seek policies that maximize \eqref{eq:HEC} as quickly as possible. Arguably the simplest approach is a greedy approach that iteratively chooses the test that increases \eqref{eq:HEC} as much as possible, in expectation over test outcomes.

Formally, we define the \textit{expected marginal gain} of a test $\actionitem$  given evidence $\subtestobs \subseteq \settestobs$ as follows:
$$\marginalfHEC(\actionitem\!\mid\!\subtestobs) \!=\! \sum_{\hypoth} \dist(\hypoth\!\mid\!\subtestobs) \Bigl(\fhec(\subtestobs \cup \{(\actionitem,\hypoth(\actionitem))\})-\fhec(\subtestobs)\Bigr)$$
Thus, $\marginalfHEC(\actionitem\!\mid\!\subtestobs)$ quantifies, for test $\actionitem$, the expected reduction in hyperedge mass upon observing the outcome of the test. Hereby, the expectation is taken w.r.t.~the distribution over hypotheses conditioned on our evidence so far. It is apparent that all hyperedges are cut if and only if $\marginalfHEC(\actionitem\!\mid\!\subtestobs)=0$ for all tests $\actionitem\in\testset$.
Given this, our \alg Algorithm simply starts with $\subtestobs=\emptyset$. It then proceeds in an iterative manner, greedily selecting the test $\actionitem^*$ that maximizes the expected marginal benefit, $\actionitem^* = \arg\max_{\actionitem} \marginalfHEC(\actionitem \mid \subtestobs)$, observes the outcome $\hypoth(\actionitem^*)$ and adds the pair $(\actionitem^*,\hypoth(\actionitem^*))$ to $\subtestobs$. It stops as soon as all edges are cut (i.e., the marginal gain of all tests is 0).

\subsection{Theoretical Analysis}
The key insight behind our analysis is that the marginal gain $\marginalfHEC(\actionitem\!\mid\!\subtestobs)$ satisfies two properties: {\em adaptive monotonicity} and {\em adaptive submodularity}, introduced by \citet{golovin_adaptive_2011} and associated with certain sequential decision problems. Formally, adaptive monotonicity simply states that the benefit of each test is nonnegative, $\marginalfHEC(\actionitem\mid\subtestobs)\geq 0$ for all tests $\actionitem\in\testset$ and evidence $\subtestobs\subseteq\settestobs$. This is straightforward, since carrying out a test can never introduce hyperedges, but only remove them. The second, slightly more subtle property -- {\em adaptive submodularity} -- states that the marginal gain of any fixed test $\actionitem\in\testset$ can never increase as we gain additional evidence. Formally, whenever $\subtestobs\subseteq\subtestobs'\subseteq\settestobs$, it must hold that $\marginalfHEC(\actionitem\mid\subtestobs)\geq \marginalfHEC(\actionitem\mid\subtestobs')$. Those properties are formally established for our $\fhec$ objective and the associated marginal gain $\marginalfHEC$ in the following Theorem:
\begin{theorem}
\label{th:as}
The objective function $\falg$ defined in \eqref{eq:HEC} is adaptive submodular and strongly adaptive monotone.
\end{theorem}
Why are these properties useful? \citet{golovin_adaptive_2011} prove that for sequential decision problems satisfying adaptive monotonicity and adaptive submodularity, greedy policies are competitive with the optimal policy. 
In particular, as a consequence of Theorem~\ref{th:as} and Theorem 5.8 of \citet{golovin_adaptive_2011}, we obtain the following result for our \alg Algorithm:
\begin{theorem}
\label{th:performance}
Assume that the prior probability distribution $\dist$ on the set of hypotheses is rational. Then, the performance of $\policyHEC$ is bounded as follows
$$ \cost(\policyHEC) \leq (\maxcard\ln(1/p_{\min})+1) \cost(\policy^*),$$
where $p_{\min} = \min_{\hypoth\in \allhypoth} \dist(h)$.
\end{theorem}

For the special case of disjoint regions (i.e., the ECD Problem, corresponding to $\maxcard=2$), our objective $\fhec$ is equivalent to the objective function proposed by \citet{golovin_bayesian_noisy_obs}, and hence our Theorem~\ref{th:performance} strictly generalizes their result. Furthermore, in the special case where each test can have at most two outcomes, and we set $\maxcard=1$, the \alg Algorithm is equivalent to the Generalized Binary Search algorithm for the ODT problem, and recovers its approximation guarantee.

\section{Efficient Implementation}\label{sec:implementation}

\begin{figure*}[t!]
\centering
 \subfigure[\emph{All Multisets}]{
\includegraphics[width=.29\textwidth]{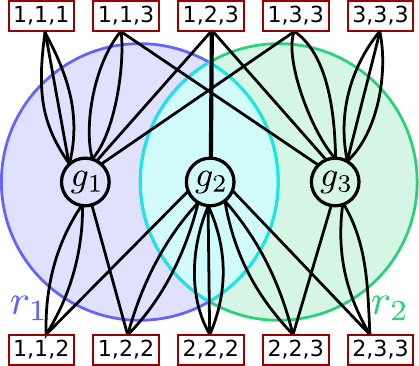}
 \label{fig:chp_all}
 }\qquad
\subfigure[\emph{$\vert\subregsettwo\vert=1$ sets removed}]{
\includegraphics[width=.29\textwidth]{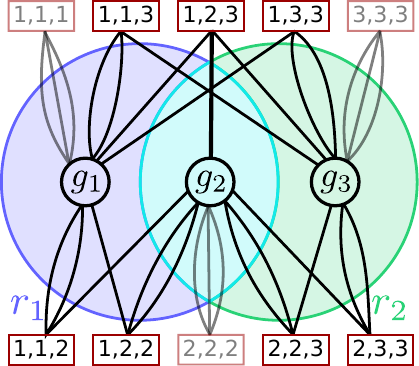}
 \label{fig:chp_rm_one}
 } \qquad
\subfigure[\emph{$\vert\subregsettwo\vert=2$ sets removed}]{
\includegraphics[width=.29\textwidth]{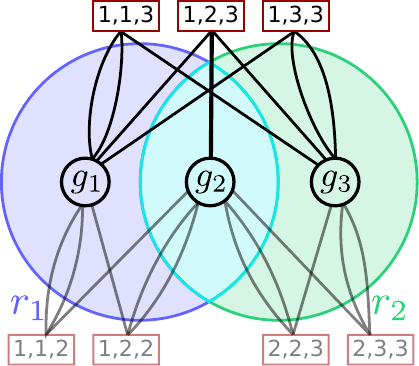}
 \label{fig:chp_rm_two}
 }\vspace{-2mm}
 \caption{\footnotesize A depiction of our algorithm as hyperedge multisets.~\subref{fig:chp_all} The equivalent hyperedges of $\comsympoly_{3}(\allsubreg)$.~\subref{fig:chp_rm_one} First iteration \algoref{alg:he_remove_regions} which removes all $\vert \subregsettwo \vert = 1$ (light edges) by subtracting $\subreg{1}\comsympoly_{2}( \{\subreg{1}\}) + \subreg{2}\comsympoly_{2}( \{\subreg{2}\}) + \subreg{3}\comsympoly_{2}( \{\subreg{3}\})$.~\subref{fig:chp_rm_two} Second iteration of \algoref{alg:he_remove_regions} which removes all $\vert \subregsettwo \vert = 2$ (light edges) by subtracting $\subreg{1}\subreg{2}\comsympoly_{1}( \{\subreg{1},\subreg{2}\}) + \subreg{2}\subreg{3}\comsympoly_{1}( \{\subreg{2},\subreg{3}\})$
 } \label{fig:hec_chp_alg} 
\end{figure*}

\begin{algorithm}[t!]
  \caption{Hyperedge Weight}
  \label{alg:he_remove_regions}
  \begin{algorithmic}
    \Procedure{Hyperedge Weight}{$\allhypoth, \maxcard$} 
    \State Compute subregions $\allsubreg$ from $\allhypoth$
    \State $W \gets \comsympoly_{\maxcard}(\allsubreg)$

    \State Initialize queue $\textrm{Q}_{1}$ with every subregion $\subreg{ } \in \allsubreg$

    \ForAll{$\cardvar \leq \maxcard$}
    \ForAll{$\subregsettwo_{\cardvar} \in Q_{\cardvar}$}
    \If{$\exists \reg{ } \text{ s.t. } \forall \hypoth \in \subregsettwo_{\cardvar}, \ \hypoth \in \reg{ } $}
    \State $W \gets W - \prod_{\subreg{ } \in \subregsettwo_{\cardvar}} p(\subregnoarg) \comsympoly_{\maxcard - \cardvar}(\subregsettwo_{\cardvar}) $
    \State Add all supersets of $\subregsettwo_{\cardvar}$ to $\textrm{Q}_{\cardvar+1}$
    \EndIf
    \EndFor
    \EndFor
    \State \textbf{return} $W$
    \EndProcedure
  \end{algorithmic}
\end{algorithm}

Our \alg algorithm computes $\marginalfHEC(\actionitem\!\mid\!\subtestobs)$ for every test in $\testset$, and greedily selects one at each time step. Naively computing this quantity involves constructing the splitting hypergraph $\hypergraphsplit$ for every possible observation, and summing the edge weights. This is computationally expensive, as constructing the graph requires enumerating every multiset of order $\maxcard$ and checking if any region contains them all, resulting in a runtime of $\bigo{|\allsubreg|^\maxcard}$. We can, however, quickly prune checks and iteratively consider multisets of growing cardinality during our computation by utilizing the following fact:

\begin{proposition}
  \label{prop:subregionsets}
  A set of subregions $\subregset$ shares a region only if all subsets $\subregset' \subset \subregset$ also share that region.
\end{proposition}

\subsection{Utilizing Complete Homogeneous Symmetric Polynomials}
Our general strategy will be to compute the sum of weights over \emph{all} multisets of cardinality $\maxcard$, and subtract those that correspond to a shared region. To do so efficiently, we identify algebraic structure in computing a sum of multisets, where a multiset corresponds to a product. Namely, it is equivalent to computing a complete homogeneous symmetric polynomial.

For any $\subregset \subseteq \allsubreg$ and cardinality $\cardvar$, we define $\allgroups_\cardvar(\subregset)$ as all multisets over groups $\subregset$ of cardinality $\cardvar$. Unlike hyperedges, these multisets can share a region. Formally
\begin{align*}
  \allgroups_\cardvar(\subregset) &= \left\{ \hyperedgeelementscardvar \subseteq \allgroups\right\}
\end{align*}
Recall that $\edgeweightfunc(\allgroups_\cardvar(\subregset)) = \sum_{\allsubreg_\cardvar(\subregset) } \prod_{\subregnoarg} \dist(\subregnoarg)$. Computing $\edgeweightfunc(\allgroups_\cardvar(\subregset))$ can be performed efficiently as this quantity is exactly equivalent to the \emph{complete homogeneous symmetric polynomial} (CHP) of degree $\cardvar$ over $\subregset$. We will briefly review a well known variant of the Newton-Girard formulae which will make an efficient algorithm for computing $\edgeweightfunc(\allgroups_\cardvar(\subregset))$ clear.

Define any set of variables $ \mathbf{x} = \{x_1, \cdots, x_n\}$. 
\begin{align*}
  \powersum_i(\mathbf{x}) &= \sum_{x \in \mathbf{x}} x^i\\
  \comsympoly_i(\mathbf{x}) &= \sum_{l_1 + \dots l_{n} = i; l_j \geq 0} \prod_{x_j \in \mathbf{x}} x_j^{l_j}
\end{align*}
Here $\powersum_i$ is the i-th \emph{power sum}, and $\comsympoly_i$ is the i-th complete homogeneous symmetric polynomial.

We have the identity~\citep{macdonald_1998_symmetric,seroul_2000_programming}:
\begin{align*}
  \comsympoly_i(\mathbf{x}) &= \frac{1}{i} \sum_{j=1}^i \comsympoly_{i - j}(\mathbf{x}) \powersum_{j}(\mathbf{x})
\end{align*}
Thus, we iteratively compute $\comsympoly_{1}(\subregset) \! \dots \! \comsympoly_{\cardvar}(\subregset)$ to compute $\edgeweightfunc(\allgroups_\cardvar(\subregset)) = \comsympoly_{\cardvar}(\subregset)$ with runtime $\bigo{\cardvar |\subregset|}$.

We now turn our attention to efficiently computing the weight of all multisets that correspond to subregions encapsulated by a region. Let $\subregsettwo$ be a set (not multiset) of subregions that shares a region. Formally: 
\begin{align*}
    \subregsettwo &= \{ \subreg{1} \dots \subreg{{\cardvar} } \} \qquad \cardvar \leq \maxcard, \nexists \regnoarg  \text{ s.t. }  \subregsettwo \subseteq \regnoarg
\end{align*}
We compute the term corresponding to $\subregsettwo$ we want to subtract from $\comsympoly_{\maxcard }(\allsubreg)$ when $\subregsettwo$ shares a region . To avoid double counting, we want the polynomial to include $\prod_{\subregnoarg \in \subregsettwo} p(\subregnoarg)$ as a factor, i.e. if we think of a hyperedge as a product, we force one link to each element of $\subregsettwo$.
\begin{align*}
  \edgeweightfunc(\subregsettwo) &=  \prod_{\subregnoarg \in \subregsettwo} p(\subregnoarg)  \sum_{l_1 + \dots l_{\cardvar} = \maxcard-\cardvar; l_i > 0} p(\subreg{1})^{l_1} \dots p(\subreg{_{\cardvar}})^{l_\cardvar} \\
  &=  \prod_{\subregnoarg \in \subregsettwo} p(\subregnoarg) \comsympoly_{\maxcard - \cardvar}(\subregsettwo)
\end{align*}
Using this, we compute $\edgeweightfunc(\hyperedgeset) = \comsympoly_{\maxcard }(\allsubreg) - \sum_{\subregsettwo \subseteq \allsubreg} \edgeweightfunc(\subregsettwo)$ by finding every set $\subregsettwo \subseteq \allsubreg$ that shares a region. Furthermore, we can utilize \propref{prop:subregionsets} to prune sets, and only consider $\subregsettwo_{\cardvar+1}$ which are supersets of any $\subregsettwo_{\cardvar}$. The algorithm is detailed in \algoref{alg:he_remove_regions}, and depicted in \figref{fig:hec_chp_alg}.

Additionally, we note that region assignments do not change as observations are received. In practice, we find all sets of subregions that share a region once. At each time step, we need only sum over the terms corresponding to remaining hypotheses.

Note that in the worst case, this algorithm still has complexity $\bigo{|\allsubreg|^\maxcard}$. This occurs when many, at least $\maxcard$, subregions share a single region. The complexity is then controlled by how many distinct subregions a single region can be shattered into, and the largest number of regions a single hypothesis can belong to. However, for many practical problems, we might expect many regions to be separated, e.g., when $|\allreg| \gg \maxcard$. In this case, \algoref{alg:he_remove_regions} will be significantly more efficient. 

Finally, we note that we can utilize an \emph{accelerated adaptive greedy} algorithm applicable to all adaptive submodular functions, which directly uses the diminishing returns property to skip reevaluation of actions~\citep{golovin_adaptive_2011}.

%
%
%
%


\newcommand{\mldata}{\texttt{MovieLens 100k}\xspace}
\newcommand{\rbdata}{\texttt{Robot}\xspace}

\newcommand{\GBS}{GBS\xspace}
\newcommand{\GBShec}{GBS-HEC\xspace}
\newcommand{\EC}{EC2\xspace}
\newcommand{\EChec}{EC2-HEC\xspace}
\newcommand{\VOI}{VoI\xspace}

\section{Experiments}
In this section, we empirically evaluate \alg on the two applications - approximate comparison-based learning and touch based localization with a robotic end effector. 

We compare \alg with five baselines. The first two are variants of algorithms for the specialized versions of the \probacro problem described earlier - generalized binary search~\citep{nowak_noisy_gbs} and equivalence class edge cutting~\citep{golovin_bayesian_noisy_obs}. For generalized binary search (\GBS), we assign each hypothesis to its own decision region, and run \alg on this hypothesis-region assignment until only one hypothesis remains. To apply equivalence class edge cutting (\EC), decision regions must be disjoint. Thus, we randomly assign each hypothesis to one of the decision regions that it belongs to, and run \EC until only one of these new regions remains. For each of these, we also run a slightly modified version, termed \GBShec and \EChec respectively, which selects tests based on these algorithms, but terminates once all hypotheses are contained in one decision region in the original \probacro problem (i.e. when the \alg termination condition is met).


The last baseline is a classic heuristic from decision theory: myopic value of information (\VOI)~\citep{howard66voi}. We define a utility function $U(\hypoth,\regnoarg)$ which is 1 if $\hypoth\in\regnoarg$ and 0 otherwise. The utility of $\version(\subtestobs)$ corresponds to the maximimum expected utility of any decision region, i.e. the expected utility if we made a decision now. \VOI greedily chooses the test that maximizes (in expectation over observations) the gain in this utility. Note that if we could solve the intractable problem of nonmyopically optimizing \VOI (i.e., look ahead arbitrarily to consider outcomes of sequences of tests), we could solve the \probacro problem optimally. In some sense, \alg can be viewed as a surrogate function for nonmyopic value of information.


\begin{figure*}
\centering
\subfigure[\emph{\mldata} ($\maxcard = 3$)]{
\vtop{
  \vskip0pt
  \hbox{%
\includegraphics[width=.302\textwidth]{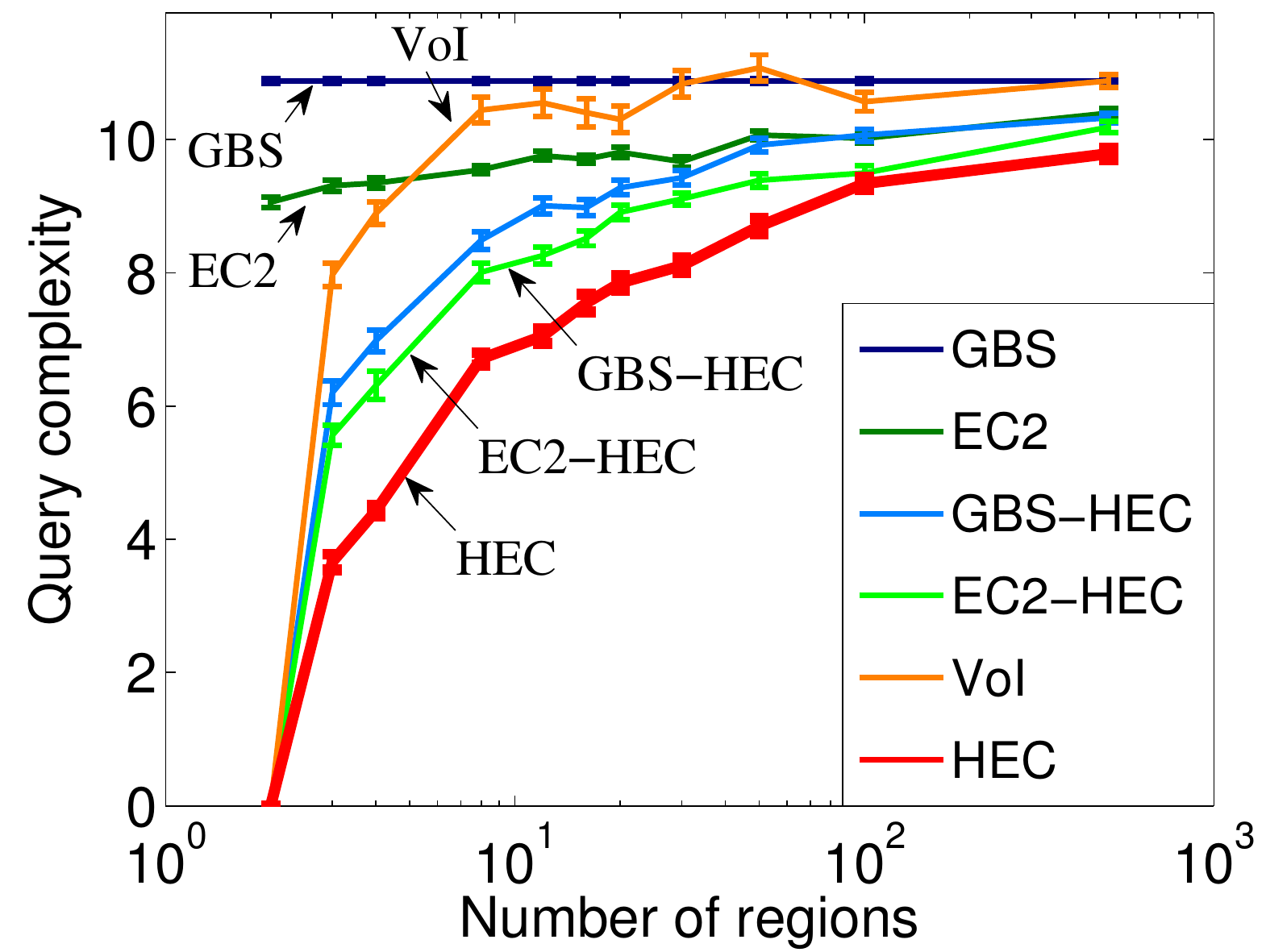}
  }%
}
 \label{fig:ml-plot-standalone}
 }
 \subfigure[\emph{\mldata} ($\numreg=12$)]{

\vtop{
  \vskip0pt
  \hbox{%
\includegraphics[width=.3\textwidth]{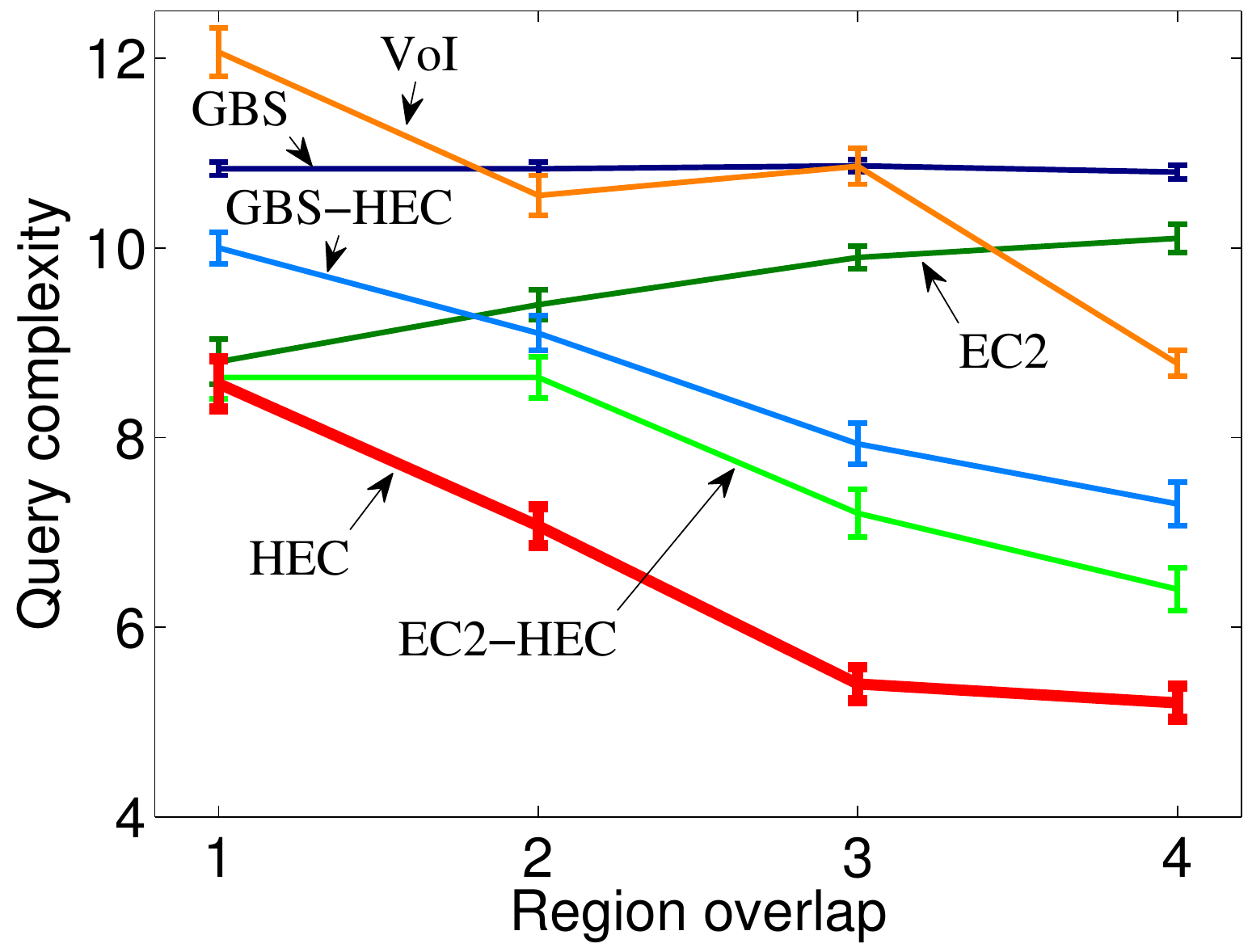}
  }%
}
 \label{fig:ml-barchart-standalone}
 } 
\subfigure[\emph{Robotic manipulation simulation}]{
\vtop{
  \vskip0pt
  \hbox{%
\includegraphics[width=.28\textwidth, trim=-3 -29 -3 1, clip=true]{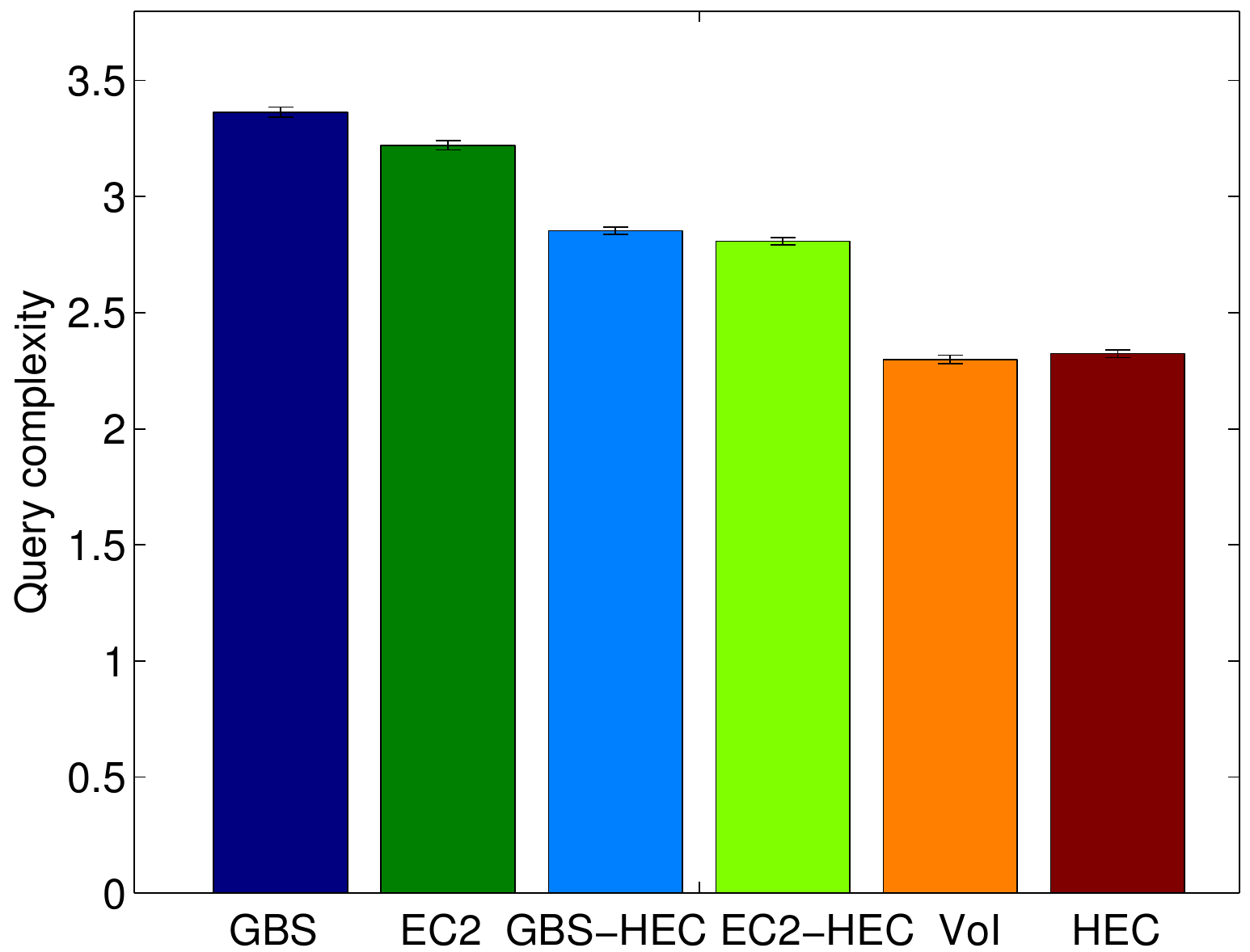}
  }%
}
 \label{fig:robot-barchart}
 }
 \vspace{-2mm}
 \caption{\footnotesize  Results on \mldata and \texttt{Robot} experiments.~\subref{fig:ml-plot-standalone} Performance as we vary the number of regions $\numreg$.~\subref{fig:ml-barchart-standalone} Performance as we vary the cardinality $\maxcard$.~\subref{fig:robot-barchart} Average performance of different algorithms across button push instances.} \label{fig:vis-robot} 
\end{figure*} 

\subsection{Approximate comparison-based learning}
We evaluate \alg on the \mldata\footnote{http://www.grouplens.org/datasets/movielens/} dataset, which consists of 1 to 5 ratings of 1682 movies from 943 users. We partition movies into decision regions using these ratings, with the goal of recommending any movie in a decision region. In order to get a similarity measurement between movies, we map them into a 10-dimensional feature space by computing a low-rank approximation of the user/rating matrix through SVD. We then use k-means to partition the set of movies into $\numreg$ (non-overlapping) clusters, corresponding to decision regions. Each movie is then assigned to the $\alpha$ closest cluster centroids. See \figref{fig:mlvisualization} for an illustration. A test corresponds to comparing two movies, an observation to selecting one of the two, and consist hypotheses are those which are closer to the selected movie (euclidean distance in 10-dimensional feature space).

Each experiment corresponds to sampling one movie as the ``true'' movie. As the number of regions increases, the size of each decision region shrinks. The size of a decision region determines how close our solution is to this (exact) target hypothesis. As a result, the problem requires the selected movie be closer to the true target, at the expense of increased query complexity. \figref{fig:ml-plot-standalone} shows the query complexity of different algorithms as a function of the number of regions, with the cardinality of the \alg hypergraph fixed to $\maxcard=3$ (i.e., each hypothesis belongs to two decision regions). An extreme case is when there are only two regions and all hypotheses belong to both regions, giving a query complexity of 0. Other than that, we see that \alg performs consistently better than other methods (e.g., to identify the true region out of 8 regions, it takes on average 6.7 queries for \alg, as opposed to 8 queries for \EChec, 8.5 queries for \GBShec, and 10.3 queries for \VOI).

\begin{figure}[t!]
\centering
\subfigure[\emph{Partitions ($\maxcard=2$)}]{
\includegraphics[width=.22\textwidth]{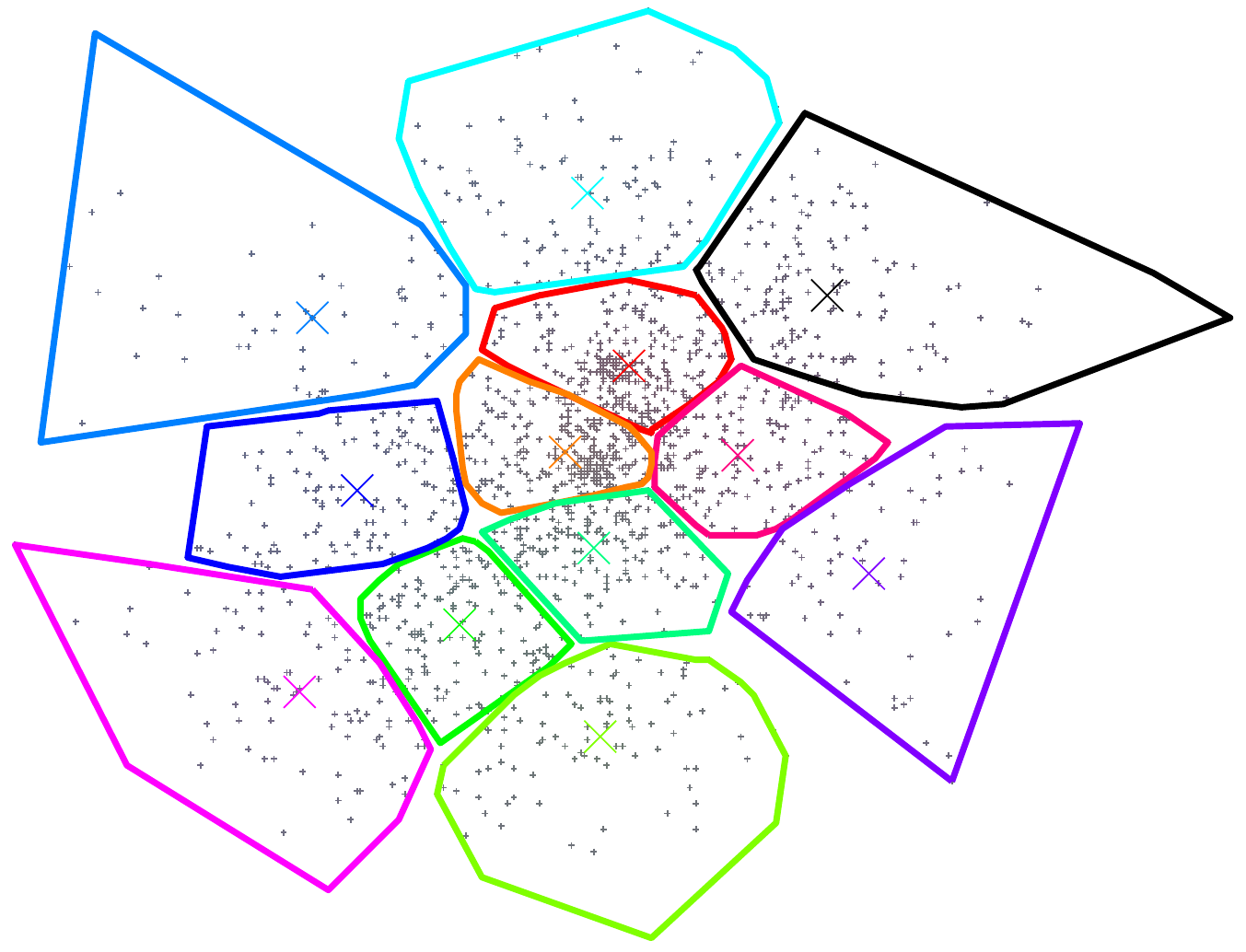}
 \label{fig:ml-v1}
 }
\subfigure[\emph{Decision regions ($\maxcard=3$)}]{
\includegraphics[width=.22\textwidth]{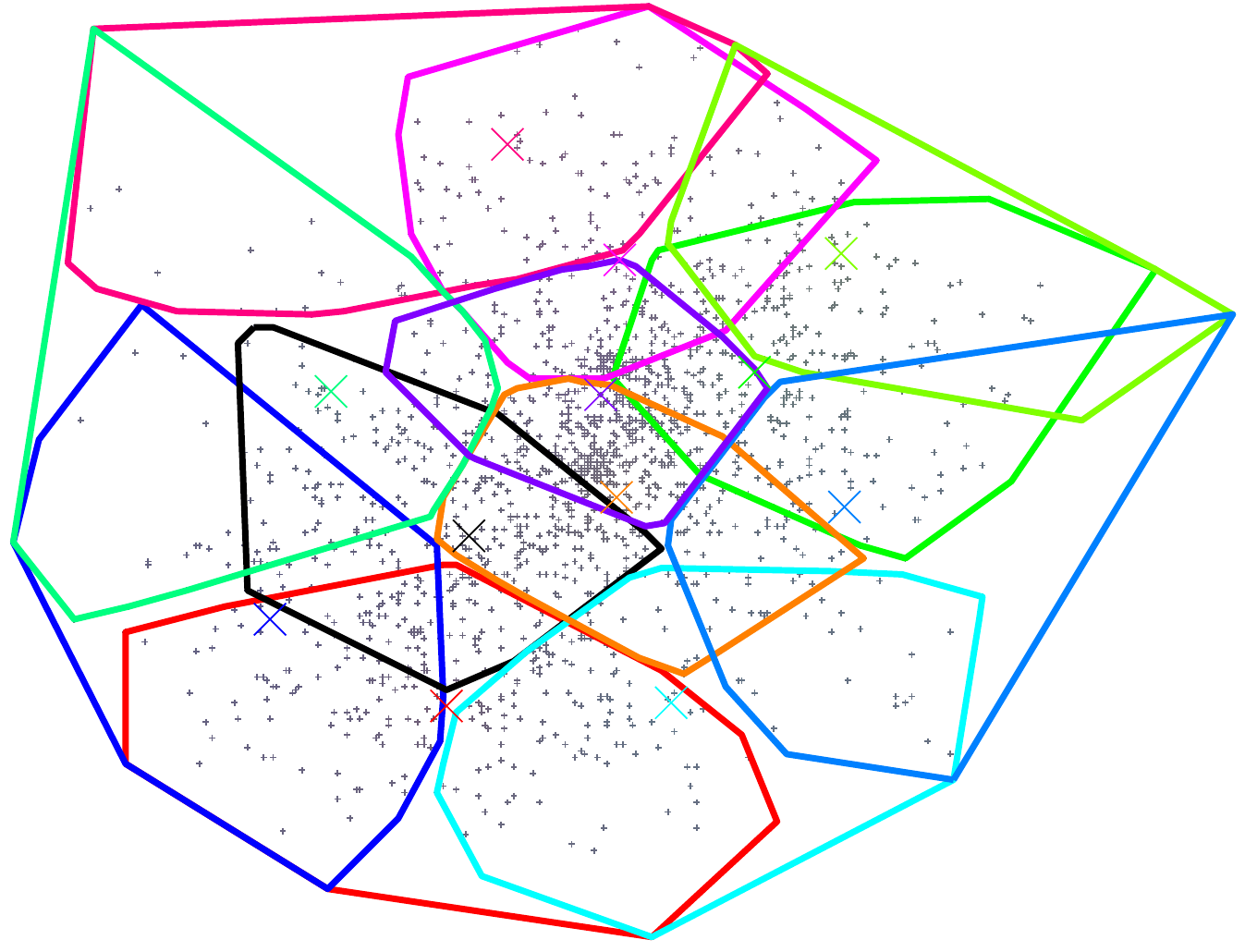}
 \label{fig:ml-v2}
 }
 \caption{\footnotesize  A 2-d illustration of (overlapping) decision regions for \mldata experiments. Dots represent movies, cross markers represent cluster centroids, and colored polygons represent decision region boundaries.~\subref{fig:ml-v1} Movies are partitioned into 12 disjoint clusters.~\subref{fig:ml-v2} Each movie is assigned to the two closest centroids.} \label{fig:mlvisualization}
\end{figure}

To see how the cardinality and region overlap influence performance, we compare the query complexity of different algorithms by varying the number of regions each hypothesis is assigned to. If we assign more regions to a hypothesis, then the search result is allowed to be further away from the true target, and thus the number of queries required for approximated search should be smaller. \figref{fig:ml-barchart-standalone} demonstrates such an effect. We fix the number of clusters to 12, and vary the number of assigned regions (and thus the hyperedge cardinality) from 1 to 4 ($\maxcard$ from 2 to 5, respectively). We see that higher cardinality enables \alg to saves more queries. For $\maxcard = 5$, it takes \alg 5.3 queries to identify a movie, whereas \VOI, \GBShec, and \EChec took 8.8, 7.4, and 6.4 queries, respectively. Additionally, \tabref{tab:mlruntime} shows the running time of \alg for these instances. We see that the accelerated implementation described in \sref{sec:implementation} enables \alg to run efficiently with reasonable hyperedge cardinality on this data set.


\begin{table}[t]
\label{tab:mlruntime} 
\begin{center}
  \begin{tabular}{ccccc}
    \toprule
$\maxcard$  &  2  & 3 & 4 & 5 \\ 
\midrule
t(\alg) & 0.026s & 0.071s   & 2.5s & $<$ 2min   \\
    \bottomrule
  \end{tabular}
  \vspace{-2mm}
\end{center}
\caption {\footnotesize Running time of \alg on \mldata with different cardinality $\maxcard$ ($\numreg$ = 12)} 
\end{table}


\begin{figure*}[t!]
\centering
 \subfigure[Hypotheses]{
\includegraphics[width=.3\textwidth, trim=200 220 200 90,clip=true]{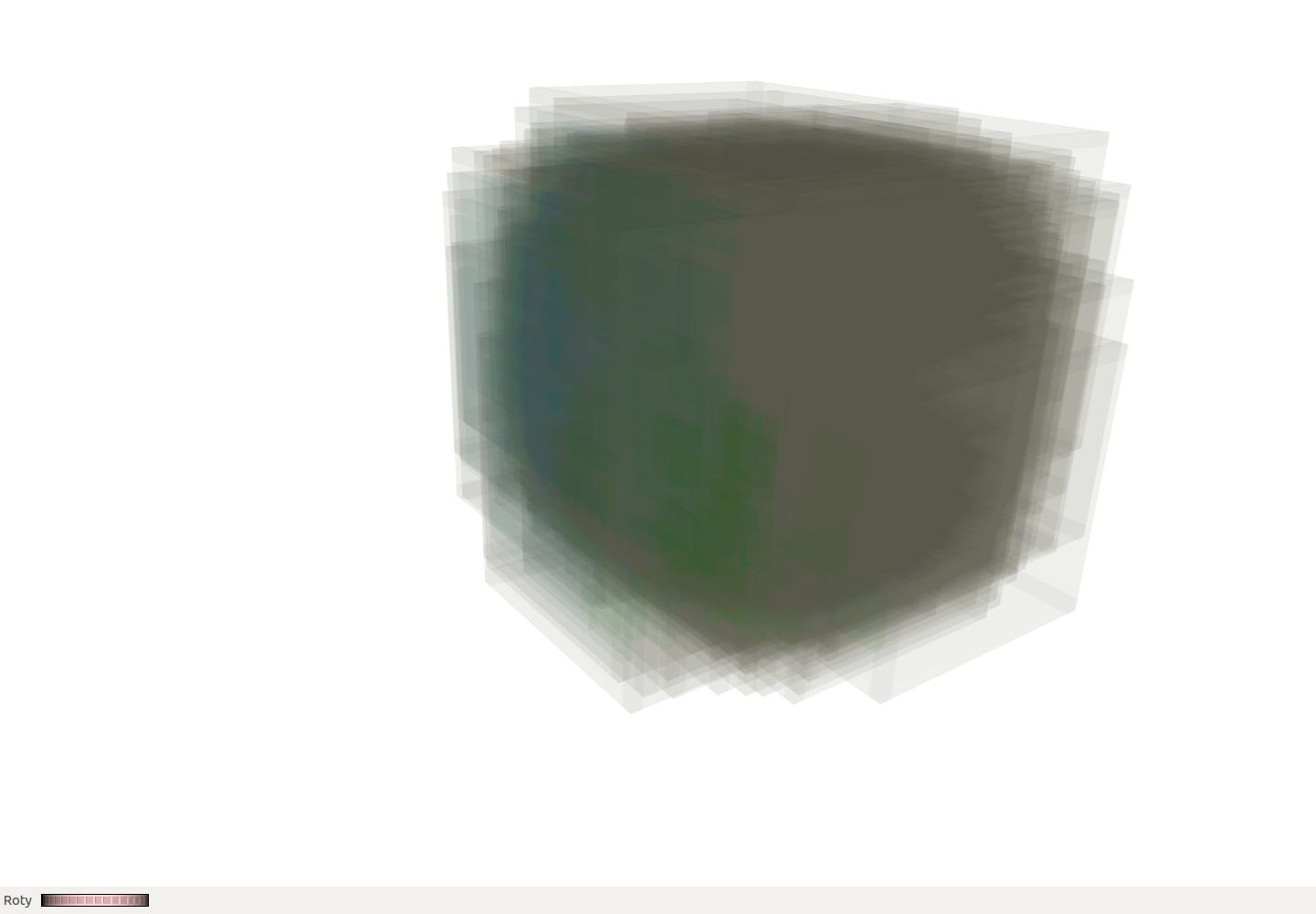}
 \label{fig:rob_parts}
 }
 \subfigure[A decision region]{
\includegraphics[width=.3\textwidth, trim=200 220 200 90,clip=true]{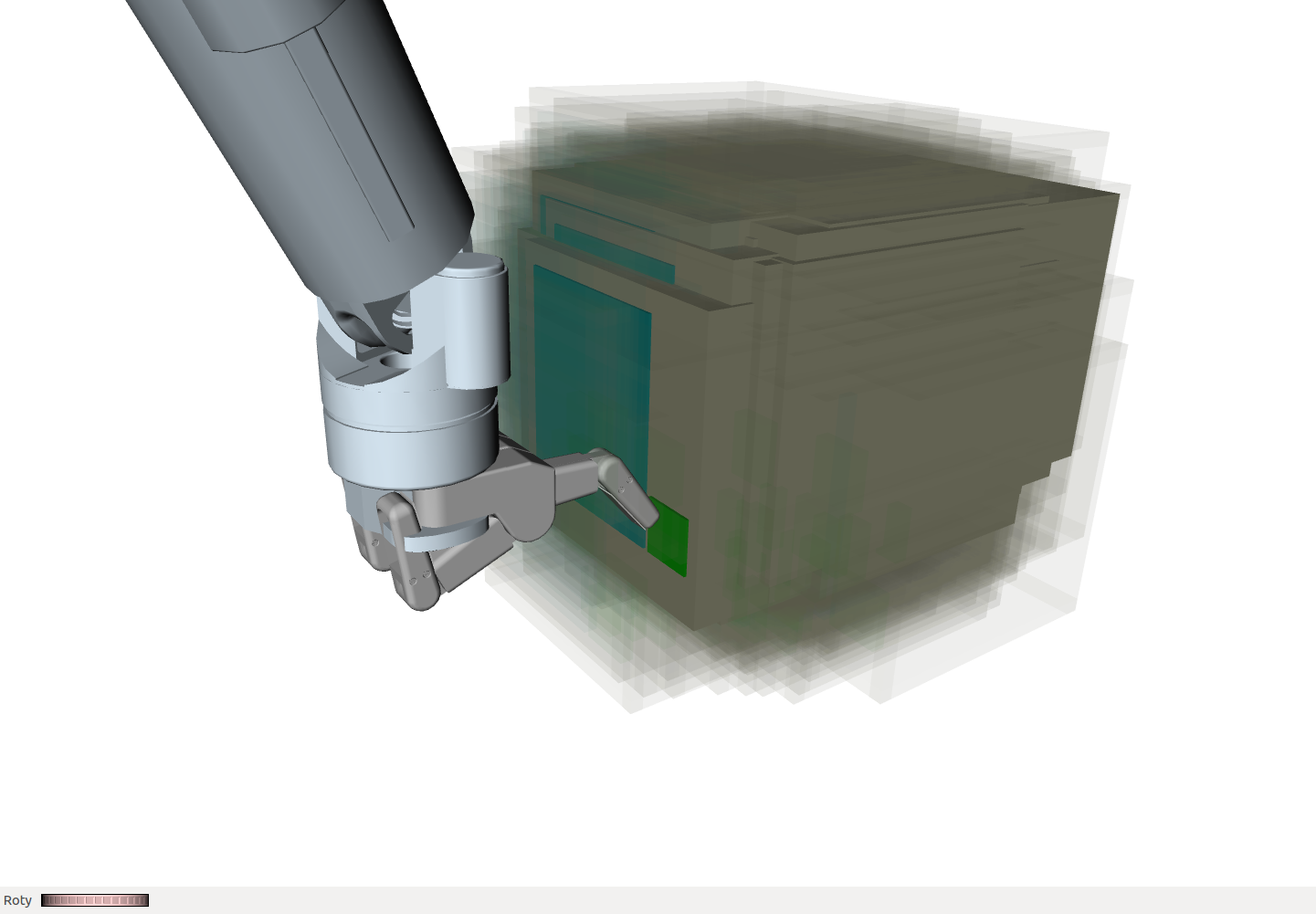}
 \label{fig:rob_region}
 }\vspace{-2mm}
 \subfigure[Two regions]{
\includegraphics[width=.3\textwidth, trim=200 220 200 90,clip=true]{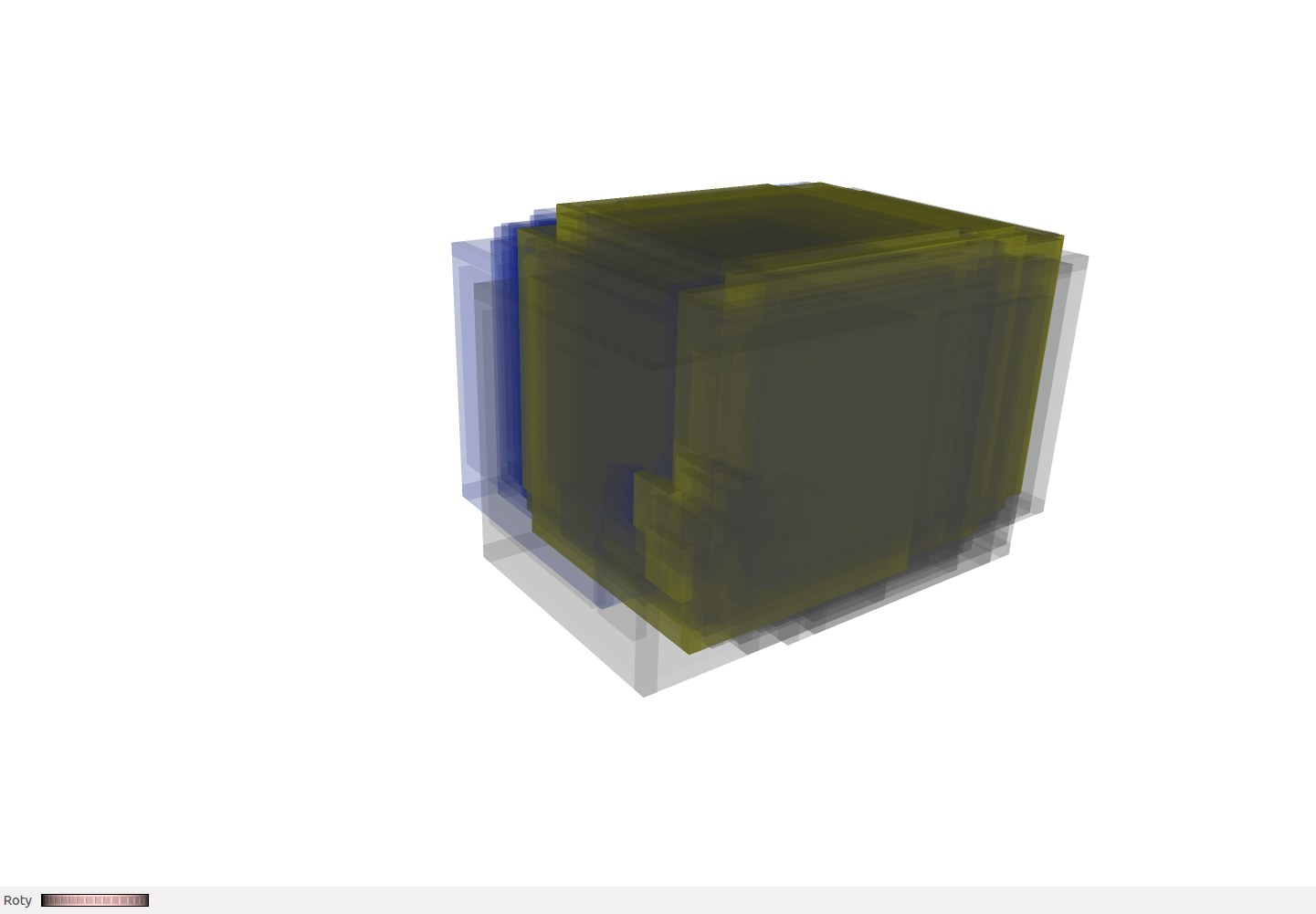}
 \label{fig:rob_region_overlap}
 }
 \caption{\footnotesize Touch based localization for pushing the button of a microwave. Given hypotheses over object location~\subref{fig:rob_parts}, decision actions are generated. The corresponding decision regions are computed by forward simulating to find hypotheses for which it would succeed~\subref{fig:rob_region}. Decision regions will overlap. In~\subref{fig:rob_region_overlap}, we see two regions (blue and grey) and their overlap (yellow).} 
 \label{fig:rob_expr_setup}
\end{figure*} 

\subsection{Touch Based localization}

We evaluate \alg on a simple robotic manipulation example. Our task is to push a button with the finger of a robotic end effector. Given a distribution over object location, we generate a set of decisions, corresponding to the end effector going to a particular pose and moving forward in a straight line. Each of these decisions will succeed on a subset of hypotheses, corresponding to a decision region. Decision regions may overlap, as a button can be pushed with many decision actions. See \figref{fig:rob_expr_setup}.

All hypotheses are not contained in a single decision region, so we perform tests to reduce uncertainty. These tests correspond to \emph{guarded moves}~\citep{will_1975_guarded}, where the end effector moves along a path until contact is sensed. After sensing contact, hypotheses are updated by eliminating object locations which could not have produced contact, e.g., if they are far away. Our goal is to find the shortest sequence of tests such that after performing them, there is a single button-push decision that would succeed for all remaining hypotheses.

Given some object location $X_{s}$, we generate an initial set of 2000 hypotheses $\allhypoth$ by sampling from $N(\mu, \Sigma)$ with $\mu = X_{s}$, and $\Sigma$ a diagonal matrix with $\Sigma_{xx} = \Sigma_{yy} = \Sigma_{zz} = 0.04$. The robot generates 50 decision regions by picking different locations and simulating the end effector forward, and noting which object poses it would succeed on. Hypotheses range from being in zero decision regions to 6, giving us a cardinality $\maxcard = 7$. For tests, the robot generates 150 guarded moves by sampling a random start location and orientation. 

We conduct experiments on 10 random environments, and randomly sample 100 hypotheses to be the ``true'' object location (for producing observations during execution), for a total of 1000 experiments. \figref{fig:robot-barchart} shows the query complexity of different algorithms averaged over these instances. We see that \alg performs well, outperforming \GBS, \GBShec, \EC, and \EChec handily. Note that myopic \VOI performs essentially the same as HEC on these experiments. This is likely due to the short horizon, where 2-3 actions were usually sufficient for reducing uncertainty to a single decision region. We would expect that for longer horizons, myopic \VOI would not perform as well.

\section{Conclusions}
In this paper, we have addressed the problem of active learning in order to facilitate decision making. We defined the Decision Region Determination (\probacro) problem, requiring that at the end of information gathering, all remaining hypotheses are confined within a single decision region (\emph{i.e.}, do not require further distinction from a decision making point of view). To address this problem, we proposed an equivalent representation in terms of a hypergraph. We prove that eliminating all edges in this hypergraph is a necessary and sufficient condition for success, suggesting a natural objective function. We show that this objective satisfies adaptive monotonicity and adaptive submodularity. This insight enabled us to prove that a greedy policy for removing hyperedges (\alg) has an approximation guarantee compared to the optimal policy. Finally, we note that at each iteration, we compute a particular polynomial, and can utilize a faster algorithm through efficient computations of complete homogeneous symmetric polynomials.

While our algorithm enables us to tackle problems of reasonable size, our computation is still exponential in hyperedge cardinality $\maxcard$. Additionally, our current scheme assumes \emph{noise-free} observations, where a hypothesis deterministically maps a test to an observation. We hope to alleviate these limitations in future work.



\subsubsection*{Acknowledgements}\label{sec:acknowledgements}

This work was supported in part by the Intel Embedded Computing ISTC, NSF Grant No. 0946825, NSF-IIS-1227495, DARPA MSEE FA8650-11-1-7156, ERC StG 307036, and a Microsoft Research Faculty Fellowship.
%

\bibliography{refs}
\bibliographystyle{abbrvnat}


\newcommand{\inreg}[1]{n_{#1}}
\newcommand{\inreghat}[1]{\widehat{n}_{#1}}
\newcommand{\numwithobs}[1]{n^{#1}}
\newcommand{\numwithobshat}[1]{\widehat{n}^{#1}}
\newcommand{\inregandobs}[2]{n_{#1}^{#2}}
\newcommand{\inregandobshat}[2]{\widehat{n}_{#1}^{#2}}
\newcommand{\totalhypoths}{N}
\newcommand{\totalhypothshat}{\widehat{N}}

\newcommand{\edgermind}{l}
\newcommand{\inregrm}{\inreg{\edgermind}}
\newcommand{\subregrm}{\subreg{\edgermind}}

\newcommand{\hyperedgesetk}{\hyperedgeset_\edgermind}
\newcommand{\hyperedgesetnok}{\overline{\hyperedgeset_{\edgermind}}}

\newcommand{\minhyperedgeset}{\hyperedgeset^{\min}}
\newcommand{\minhyperedgesetk}{\minhyperedgeset_\edgermind}
\newcommand{\minhyperedgesetnok}{\overline{\minhyperedgeset_{\edgermind}}}
\newcommand{\nominhyperedgeset}{\hyperedgeset^{\overline{\min}}}
\newcommand{\nominhyperedgesetk}{\nominhyperedgeset_\edgermind}
\newcommand{\nominhyperedgesetnok}{\overline{\nominhyperedgeset_{\edgermind}}}
\newcommand{\numk}{ {\vert\hyperedge_\edgermind\vert} }

\newcommand{\sumhe}{\sum_{\hyperedge \in \hyperedgeset}}
\newcommand{\sumhek}{\sum_{\hyperedge \in \hyperedgesetk}}
\newcommand{\sumhenok}{\sum_{\hyperedge \in \hyperedgesetnok}}
\newcommand{\summinhenok}{\sum_{\hyperedge \in \minhyperedgesetnok}}
\newcommand{\sumnominhenok}{\sum_{\hyperedge \in \nominhyperedgesetnok}}
\newcommand{\sumallobs}{\sum_{\observationitem \in \observationset}} 
\newcommand{\sumobsnotc}{\sum_{\observationitem \in \observationset \backslash c}} 

\newcommand{\prodedge}{\prod_{i \in \hyperedge}}
\newcommand{\prodedgenok}{\prod_{i \in \hyperedge, i \neq l}}

\clearpage
\onecolumn

\section{Appendix}
In this section, we provide proofs for theorems stated throughout the paper.

\subsection{$\maxcard$ for bounds}
We start by showing that for a properly defined $\maxcard$, the \probacro problem is solved ($\version(\subtestobs)\subseteq \regnoarg$) if and only if the \alg objective is maximized. However, we sometimes require a slightly greater $\maxcard$ to ensure the objective $\fhec$ is adaptive submodular. We define these below.

Let $\regset$ be a set of regions, the length of which is related to $\maxcard$. To get equivalence of the \probacro and \alg, we require that for every region in $\regset$, there is some hypothesis in all but one region of $\regset$.
\begin{align*}
  \regsetiff &= \argmax_{\regset} \vert\regset\vert \quad \text{ s.t. }\forall \reg{ } \in \regset, \exists \hypoth : \hypoth \notin \reg{ }, \hypoth \in \regset \backslash \reg{ } \\
  \maxcardiff &= \vert\regsetiff\vert
\end{align*}

Sometimes, this is not sufficient for adaptive submodularity. For this, we also require that there is some hypothesis in every region of $\regset$, and we also add one to the length of $\regset$.
\begin{align*}
  \regsetas &= \argmax_{\regset} \vert\regset\vert \quad \text{ s.t. } \circled{1} \exists \tilde{\hypoth} \in \regset \quad \circled{2} \forall \reg{ } \in \regset, \exists \hypoth : \hypoth \notin \reg{ }, \hypoth \in \regset \backslash \reg{ } \\
  \maxcardas &= \vert\regsetas\vert + 1
\end{align*}

Before moving on, we prove that $\maxcardas \geq \maxcardiff$. 
\begin{proposition}
  \label{prop:maxcard_as_iff}
  $\maxcardas \geq \maxcardiff$
\end{proposition}
\begin{proof}
  There are two cases:
  \begin{enumerate}
    \item $\exists \hypoth \in \regsetiff$. In this case, $\regsetas = \regsetiff$ and $\maxcardas = \vert\regsetas\vert + 1 = \maxcardiff + 1$.
    \item $\not\exists \hypoth \in \regsetiff$. Define $\widetilde{\regset} = \regsetiff \setminus \reg{ }$ for some $\reg{ } \in \regsetiff$. We know by definition of $\regsetiff$ that $\exists \hypoth \in \widetilde{\regset}$. Additionally, we know by definition of $\maxcardiff$ that $\forall \reg{ } \in \widetilde{\regset}, \exists \hypoth, \hypoth \notin \reg{ }, \hypoth \in \regsetiff \setminus \reg{ }$, so it follows that $\hypoth \in \widetilde{\regset} \setminus \reg{ }$. Therefore, we know $\widetilde{\regset}$ satisfies the constraints for $\regsetas$, and $\maxcardas \geq \vert\widetilde{\regset}\vert+1 = \vert\regsetiff\vert = \maxcardiff$.
  \end{enumerate}
\end{proof}

Our algorithm actually utilizes $\displaystyle \maxcard = \min \left(\max_{\hypoth\in\allhypoth}\vert\{\regnoarg:\hypoth\in\regnoarg\}\vert, \max_{\regnoarg \in \allreg} \vert\{\subregnoarg : \subregnoarg \in \regnoarg\}\vert\right) + 1$. We briefly show that each of these also upper bound $\maxcardas$.

\begin{proposition}
  \label{prop:maxcard_maxreg_bound}
  $\max_{\hypoth\in\allhypoth}\vert\{\regnoarg:\hypoth\in\regnoarg\}\vert+1 \geq \maxcardas$
\end{proposition}
\begin{proof}
 Note that condition $\circled{1}$ in $\regsetas$ bounds $\vert\regsetas\vert$ by $\max_{\hypoth\in\allhypoth}\vert\{\regnoarg:\hypoth\in\regnoarg\}\vert$. The result follows.
\end{proof}
 
\begin{proposition}
  \label{prop:maxcard_maxsubreg_bound}
  $\max_{\regnoarg \in \allreg} \vert \{\subregnoarg : \subregnoarg \in \regnoarg\}\vert+ 1 \geq \maxcardas$
\end{proposition}
\begin{proof}
    Let $\regnoarg$ be an element of $\regsetas$. By definition, it is required that at least $\vert\regsetas\vert$ different subregions $\subregnoarg_{1} \cdots \subregnoarg_{\vert\regsetas\vert}$ be in that region - one which is in every other region in $\regsetas$ to satisfy condition $\circled{1}$, and $\vert\regsetas\vert-1$ which are in all but one of the $\regsetas-1$ other regions to satisfy condition $\circled{2}$. The result follows.
\end{proof}

Thus, we can utilize $\displaystyle \maxcard = \min \left(\max_{\hypoth\in\allhypoth}\vert\{\regnoarg:\hypoth\in\regnoarg\}\vert, \max_{\regnoarg \in \allreg} \vert \{\subregnoarg : \subregnoarg \in \regnoarg \}\vert\right) + 1$ and apply the proofs using cardinality at least $\maxcardas$ and $\maxcardiff$. While our bounds and algorithm are better if we knew the correct $\maxcardas$ to use, finding that value is itself hard to compute - thus, our implementation uses the value defined in \sref{sec:results} and copied above.

\subsection{\thmref{theory:hyperedge_iff}: Equivalence of \probacro and \alg}
\begin{proof}
  We first prove that if all $\hypoth$ are contained in one region, then all edges are cut, i.e. $\exists \regnoarg:\version(\subtestobs)\subseteq\regnoarg \Rightarrow \hyperedgesetobs=\emptyset$. This is by construction, since a hyperedge $\hyperedge \in \hyperedgesetobs$ is only between subregions (or hypotheses) that do not share any regions. More concretely, our definition of $\hyperedge$ requires $\not\exists \regnoarg \text{ s.t. } \forall \hypoth \in \hyperedge : \hypoth \in \regnoarg$. Since all remaining nodes $\version(\subtestobs) \subseteq \regnoarg$, there will be no such such set of hypotheses.

  Next, we prove that if all edges are removed, then all $\hypoth$ are contained in one region, i.e., $\hyperedgesetobs=\emptyset \Rightarrow \exists \regnoarg:\version(\subtestobs) \subseteq\regnoarg$. Clearly, if we set $\vert\version(\subtestobs)\vert \leq \maxcard$, this condition would be met - $\hyperedgesetobs$ would check every subset of $\version(\subtestobs)$ to see if they shared a region, and would draw a hyperedge i.f.f. they do not.
  To complete the proof, we will make use of the following lemma:
  \begin{lemma}
    \label{lemma:edgeremoved}
    Define $\beta$ as some constant s.t. $\beta \geq \maxcard$. $\forall \hypothset \subseteq \allhypoth, \vert\hypothset\vert = \beta, \exists \regnoarg : \hypothset \subseteq \regnoarg \Rightarrow \forall \{\hypothset \cup \hypoth\} \subseteq \allhypoth, \exists \regnoarg : \{\hypothset \cup \hypoth\} \in \regnoarg $
  \end{lemma}

  \begin{proof}
    For the sake of contradiction, suppose $\nexists \regnoarg : \{\hypothset \cup \hypoth\} \in \regnoarg$. This must mean $\hypoth \not\in \hypothset$. Let $\{\hypothset \cup \hypoth\} = \{\hypoth_1, \hypoth_2, \dots, \hypoth_{\beta+1}\}$. Let $\hypothset_{i}$ be the subset of $\{\hypothset \cup \hypoth\}$ which does not include the $i$th $\hypoth$ from $\{\hypothset \cup \hypoth \}$, i.e. $\hypothset_{i} = \{\hypoth_1 \dots, \hypoth_{i-1}, \hypoth_{i+1}, \dots \hypoth_{\beta+1}\}$. By assumption, we know $\exists \regnoarg : \hypothset_{i} \in \regnoarg$. Let $\reg{i}$ be that region for $\hypothset_i$. If $\reg{i} = \reg{j}$, for any $i,j$, this would imply $\{ \hypothset_i \cup \hypothset_j \} = \{\hypothset \cup \hypoth\} \in \reg{i}$. Thus, each $\reg{i}$ must be unique if $\not\exists \regnoarg : \{\hypothset \cup \hypoth\} \in \regnoarg$. Furthermore, this implies $\hypoth_{i} \not\in \reg{i}$, and $\hypoth \in \reg{j}, \forall j\neq i$. Let $\regset_{\beta+1} = \{\reg{1} \dots \reg{\beta+1} \}$. By definition of $\beta$, we know $\beta \geq \maxcard \geq \maxcard_{iff}$. But this causes a contradiction - by definition of $\maxcard_{iff}$, the maximum set of regions $\regset$ where $\hypoth_{i} \not\in \reg{i}, \hypoth_{i} \in \reg{j} \forall j \neq i$ is $\maxcardiff$. But $\regset_{\beta+1}$ would require such a set of regions where $\vert\regset_{\beta+1}\vert = \beta+1 \geq \maxcardiff +1$. Thus, we have a contradiction, and have shown $\exists \regnoarg : \{\hypothset \cup \hypoth\} \in \regnoarg$.
  \end{proof}
  
  By construction, we know that if $\hyperedgesetobs = \emptyset \Rightarrow \forall \hypothset \subseteq \allhypoth, \vert\hypothset\vert \leq \maxcard, \exists \regnoarg : \hypothset \subseteq \regnoarg$. Applying \lemmaref{lemma:edgeremoved} inductively, this implies, $\forall \{\hypothset \cup \hypoth_1\} \subseteq \version(\subtestobs), \exists \regnoarg : \{\hypothset \cup \hypoth_1\} \subseteq \regnoarg \Rightarrow \forall \{\hypothset \cup \hypoth_1 \cup \hypoth_2\} \subseteq \version(\subtestobs), \exists \regnoarg : \{\hypothset \cup \hypoth_1 \cup \hypoth_2\} \subseteq \regnoarg \Rightarrow \dots \Rightarrow \exists \regnoarg : \version(\subtestobs) \subseteq \regnoarg$.
\end{proof}

\subsection{\thmref{th:as}: strong adaptive monotonicity and adaptive submodularity}
\label{sec:asproof}
\begin{proof}
We start with showing our formulation is strongly adaptive monotone.
\begin{lemma}
  \label{lemma:hec_monotone}
  The function $\fhec$ described above is strongly adaptive monotone, i.e.
  \begin{align*}
    \fhec(\subtestobs \cup \{(\actionitem,\hypoth(\actionitem))\}) - \fhec(\subtestobs) \geq 0 \quad \forall \actionitem, \hypoth
  \end{align*}
\end{lemma}
\begin{proof}
  This states that our utility function must always increase as we take additional actions and receive observations. Intuitively, we can see that additional action observation pairs can only cut edges, and thus our utility function always increases. More concretely:
  \begin{align*}
    & \fhec(\subtestobs \cup \{(\actionitem,\hypoth(\actionitem))\}) - \fhec(\subtestobs) \\
    &= \bigl(\edgeweightfunc(\hyperedgeset)-\edgeweightfunc(\hyperedgeset(\subtestobs \cup \{(\actionitem,\hypoth(\actionitem))\})) \bigr) - \bigl(\edgeweightfunc(\hyperedgeset)-\edgeweightfunc(\hyperedgeset(\subtestobs)) \bigr)\\
    &= \edgeweightfunc(\hyperedgeset(\subtestobs)) - \edgeweightfunc(\hyperedgeset(\subtestobs \cup \{(\actionitem,\hypoth(\actionitem))\})) \\
    &= \edgeweightfunc(\{\hyperedge\in\hyperedgeset : \forall (i,o)\in\subtestobs\;\forall \widetilde{\hypoth}\in\hyperedge, \widetilde{\hypoth}(i)=o\}) \\
    & \qquad - \edgeweightfunc(\{\hyperedge\in\hyperedgeset : \forall (i,o)\in\subtestobs\;\forall \widetilde{\hypoth}\in\hyperedge, \widetilde{\hypoth}(i)=o, \widetilde{\hypoth}(t) = \hypoth(t)\}) & \text{by definition of $\hyperedgeset(\subtestobs)$} \\
    &= \edgeweightfunc(\{\hyperedge\in\hyperedgeset : \forall (i,o)\in\subtestobs\;\forall \widetilde{\hypoth}\in\hyperedge, \widetilde{\hypoth}(i)=o, \widetilde{\hypoth}(t) \neq \hypoth(t)\} ) \\
    & \geq 0 \quad &\text{since $\edgeweightfunc(\hyperedge) \geq 0 \; \forall \hyperedge$}
  \end{align*}
\end{proof}

Next, we prove that our formulation is adaptive submodular:
\begin{lemma}
  \label{lemma:hec_as}
  The function $\fhec$ described above is adaptive submodular for any prior with rational values, i.e. for $\subtestobs \subseteq \widehat{\subtestobs} \subseteq \settestobs$
  \begin{align*}
    \marginalfHEC(\actionitem\!\mid\!\subtestobs) \geq \marginalfHEC(\actionitem\!\mid\!\widehat{\subtestobs}) \quad \forall \actionitem \in \testset \backslash \subtestobsactions
  \end{align*}
  where $\subtestobsactions$ are the set of tests in $\subtestobs$.
\end{lemma}
\begin{proof}
  This states that our expected utility for a fixed action $\actionitem$ decreases as we take additional actions and receive observations. We rewrite our expected marginal utility in a more convenient form:
  \begin{align*}
    \marginalfHEC(\actionitem\!\mid\!\subtestobs) &=  \sum_{\hypoth} \dist(\hypoth\!\mid\! \subtestobs) \Bigl(\fhec(\subtestobs \cup \{(\actionitem,\hypoth(\actionitem))\})-\fhec(\subtestobs)\Bigr) \\
    &=  \sum_{\hypoth} \dist(\hypoth\!\mid\!\subtestobs) \Bigl(  \left[ \edgeweightfunc(\hyperedgeset) - \edgeweightfunc(\hyperedgeset(\subtestobsunion))\right]  - \left[ \edgeweightfunc(\hyperedgeset) - \edgeweightfunc(\hyperedgeset(\subtestobs)) \right]  \Bigr) \\
    &=  \sum_{\hypoth} \dist(\hypoth\!\mid\!\subtestobs) \Bigl(  \edgeweightfunc(\hyperedgeset(\subtestobs))- \edgeweightfunc(\hyperedgeset(\subtestobsunion)) \Bigr) \\
  \end{align*}

  For convenience, we define $\inregandobs{i}{\observationitem}$ to be the total probability mass in $\subreg{i}$ consistent with all evidence in $\subtestobs$ and observation $\observationitem$. We define $\inreg{i}$ and $\numwithobs{\observationitem}$ similarly. More formally:
  \begin{align*}
    \inregandobs{i}{\observationitem} &= \sum_{\hypoth \in \subreg{i}} \dist(\hypoth) \indicator{ \hypoth \in \version(\subtestobs \cup \{(\actionitem, \observationitem)\})} \\
    \inreg{i} &= \sum_{\observationitem \in \observationset} \inregandobs{i}{\observationitem}\\
    \numwithobs{\observationitem} &= \sum_{\subreg{i} \in \allsubreg} \inregandobs{i}{\observationitem} \\
    \totalhypoths &= \sum_{\subreg{i} \in \allsubreg} \sum_{\observationitem \in \observationset} \inregandobs{i}{\observationitem}\\
    \edgeweightfunc(\hyperedgeset(\subtestobs)) &= \sum_{\hyperedge \in \hyperedgeset} \prod_{i \in \hyperedge} \inreg{i}
  \end{align*}
  Similarly, we can also write $\edgeweightfunc(\hyperedgeset(\subtestobs \cup \{(\actionitem, \observationitem)\})) = \sum_{\hyperedge \in \hyperedgeset} \prod_{i \in \hyperedge} \inregandobs{i}{\observationitem}$. We can rewrite our objective as:
  \begin{align*}
    \marginalfHEC(\actionitem\!\mid\!\subtestobs) &= \sum_{\hypoth} \dist(\hypoth\!\mid\!\subtestobs) \Bigl(  \sum_{\hyperedge \in \hyperedgeset} \prod_{i \in \hyperedge} \inreg{i} - \sum_{\hyperedge \in \hyperedgeset} \prod_{i \in \hyperedge} \inregandobs{i}{\hypoth(\actionitem)}  \Bigr) \\
    &= \sum_{\observationitem} \frac{\numwithobs{\observationitem}}{\totalhypoths} \Bigl(  \sum_{\hyperedge \in \hyperedgeset} \prod_{i \in \hyperedge} \inreg{i} - \sum_{\hyperedge \in \hyperedgeset} \prod_{i \in \hyperedge} \inregandobs{i}{\observationitem}  \Bigr) \\
    &= \sum_{\hyperedge \in \hyperedgeset} \prod_{i \in \hyperedge} \inreg{i} - \sum_{\observationitem} \frac{\numwithobs{\observationitem}}{\totalhypoths} \sum_{\hyperedge \in \hyperedgeset} \prod_{i \in \hyperedge} \inregandobs{i}{\observationitem}  
  \end{align*}

  Similarly, we define variables for the evidence $\widehat{\subtestobs}$, i.e. $\inregandobshat{i}{\observationitem}$ for the total probability mass in $\subreg{i}$ consistent with all evidence in $\widehat{\subtestobs}$ and observation $\observationitem$:
  \begin{align*}
    \marginalfHEC(\actionitem\!\mid\!\widehat{\subtestobs}) &= \sum_{\hyperedge \in \hyperedgeset} \prod_{i \in \hyperedge} \inreghat{i} - \sum_{\observationitem} \frac{\numwithobshat{\observationitem}}{\totalhypothshat} \sum_{\hyperedge \in \hyperedgeset} \prod_{i \in \hyperedge} \inregandobshat{i}{\observationitem}  
  \end{align*}

  We rewrite what we would like to show as:
  \begin{align*}
    &\marginalfHEC(\actionitem\!\mid\!\subtestobs) - \marginalfHEC(\actionitem\!\mid\!\widehat{\subtestobs}) \\
    &=  \Bigr( \sum_{\hyperedge \in \hyperedgeset} \prod_{i \in \hyperedge} \inreg{i} - \sum_{\observationitem} \frac{\numwithobs{\observationitem}}{\totalhypoths} \sum_{\hyperedge \in \hyperedgeset} \prod_{i \in \hyperedge} \inregandobs{i}{\observationitem} \Bigl) - \Bigr(  \sum_{\hyperedge \in \hyperedgeset} \prod_{i \in \hyperedge} \inreghat{i} -\sum_{\observationitem} \frac{\numwithobshat{\observationitem}}{\totalhypothshat} \sum_{\hyperedge \in \hyperedgeset} \prod_{i \in \hyperedge} \inregandobshat{i}{\observationitem} \Bigl)\\
    &\geq 0
  \end{align*}

  We will show that for any single action observation pair, which corresponds to eliminating a single hypothesis, the expected utility of a test will always decrease. General adaptive submodularity, which states the expected utility decreases with any additional evidence, follows easily. For convenience, we consider rescaling our function so that all $\inregandobs{i}{o}$ are integers, which is possible since we assumed a rational prior. Note that a function $f$ is adaptive submodular i.f.f. $cf$ is adaptive submodular for any constant $c > 0$, so showing adaptive submodularity in the rescaled setting implies adaptive submodularity for our setting.
  \begin{sublemma}
    \label{sublemma:onepartremoval}
    If we remove one hypothesis from subregion $k$ which agrees with observation $c$, i.e.
  \begin{align*}
    \widehat{\inregandobs{i}{\observationitem}} &= \begin{cases}
      \inregandobs{i}{\observationitem} - 1 &\text{if $i=\edgermind$ and $\observationitem = c$}\\
      \inregandobs{i}{\observationitem}  &\text{else}\\
    \end{cases}
  \end{align*}
  then
  \begin{align*}
    \Delta = \Bigr( \sum_{\hyperedge \in \hyperedgeset} \prod_{i \in \hyperedge} \inreg{i} - \sum_{\observationitem} \frac{\numwithobs{\observationitem}}{\totalhypoths} \sum_{\hyperedge \in \hyperedgeset} \prod_{i \in \hyperedge} \inregandobs{i}{\observationitem} \Bigl) - \Bigr(  \sum_{\hyperedge \in \hyperedgeset} \prod_{i \in \hyperedge} \inreghat{i} -\sum_{\observationitem} \frac{\numwithobshat{\observationitem}}{\totalhypothshat} \sum_{\hyperedge \in \hyperedgeset} \prod_{i \in \hyperedge} \inregandobshat{i}{\observationitem} \Bigl) &\geq 0
  \end{align*}
  \end{sublemma}
  \begin{proof}
  Based on our definitions, it follows that:
  \begin{align*}
    \widehat{\inreg{i}} &= \begin{cases}
      \inreg{i} - 1 &\text{if $i=\edgermind$}\\
      \inreg{i}  &\text{else}\\
    \end{cases} \\
    \widehat{\numwithobs{\observationitem}} &= \begin{cases}
      \numwithobs{\observationitem} - 1 &\text{if $\observationitem=c$}\\
      \numwithobs{\observationitem}  &\text{else}\\
    \end{cases} \\
    \widehat{\totalhypoths} &= \totalhypoths - 1
  \end{align*}
    
  We split the difference into three terms:
  \begin{align*}
    \Delta^a &= \sumhe \left(\prodedge \inreg{i} - \prodedge\widehat{\inreg{i}} \right)\\
    \Delta^b &= \sumobsnotc \sumhe \left( - \frac{\numwithobs{\observationitem}}{\totalhypoths}  \prodedge \inregandobs{i}{\observationitem} +  \frac{\widehat{\numwithobs{\observationitem}}}{\widehat{\totalhypoths}}  \prodedge \widehat{\inregandobs{i}{\observationitem}} \right) \\
    \Delta^c &= \sumhe \left( -\frac{\numwithobs{c}}{\totalhypoths}  \prodedge \inregandobs{i}{c} +  \frac{\widehat{\numwithobs{c}}}{\widehat{\totalhypoths}}  \prodedge \widehat{\inregandobs{i}{c}} \right) \\
    \Delta^a + \Delta^b + \Delta^c &= \Delta
  \end{align*}
    
  To aid in notation, we define $\hyperedgesetk = \{\hyperedge \in \hyperedgeset: \subregrm \in \hyperedge\}$, hyperedges that contain region $\edgermind$, and $\hyperedgesetnok = \hyperedgeset \backslash \hyperedgesetk$, all other hyperedges. Additionally, let $\numk$ be the number of times $\subregrm$ appears in the multiset $\hyperedge$.

  First term:
  \begin{align*}
    \Delta^a &= \sumhe \left( \prodedge \inreg{i} - \prodedge\widehat{\inreg{i}} \right)\\
    &= \sumhenok \left[ \prodedge \inreg{i} - \prodedge\widehat{\inreg{i}}\right] + \sumhek \left[\prodedge \inreg{i} - \prodedge\widehat{\inreg{i}} \right] \\
    &= \sumhenok \left[ \prodedge \inreg{i} - \prodedge\inreg{i}\right] + \sumhek \left[  \left(\prodedgenok \inreg{i} \right)\inregrm^\numk - \left(\prodedgenok\inreg{i} \right) (\inregrm-1)^\numk \right] \\
    &= \sumhek \left(\prodedgenok\inreg{i} \right) \left(\inregrm^\numk - (\inregrm-1)^\numk \right)\\
    &\geq 0 \quad \text{(since $\inregrm \geq 1$)} \\
  \end{align*}

    Second term:
    \begin{align*}
      \Delta^b &= \sumobsnotc \sumhe \left( -\frac{\numwithobs{\observationitem}}{\totalhypoths}  \prodedge \inregandobs{i}{\observationitem} +  \frac{\widehat{\numwithobs{\observationitem}}}{\widehat{\totalhypoths}}  \prodedge \widehat{\inregandobs{i}{\observationitem}} \right) \\
      &= \sumobsnotc \sumhe \left( -\frac{\numwithobs{\observationitem}}{\totalhypoths}  \prodedge \inregandobs{i}{\observationitem} + \frac{\numwithobs{\observationitem}}{\totalhypoths-1}  \prodedge \inregandobs{i}{\observationitem} \right) \\
      &= \sumobsnotc \sumhe \frac{\numwithobs{\observationitem}}{\totalhypoths(\totalhypoths-1)}  \prodedge \inregandobs{i}{\observationitem} \\
      & \geq 0 \quad \text{(since each term $\geq 0$)}
    \end{align*}
    
    Third term:
    \begin{align*}
      \Delta^c &= \sumhe \left( -\frac{\numwithobs{c}}{\totalhypoths}  \prodedge \inregandobs{i}{c} + \frac{\widehat{\numwithobs{c}}}{\widehat{\totalhypoths}}  \prodedge \widehat{\inregandobs{i}{c}} \right) \\
      &= -\frac{\numwithobs{c}}{\totalhypoths}  \sumhe \prodedge \inregandobs{i}{c} + \frac{\numwithobs{c}-1}{\totalhypoths-1} \left( \sumhe \left(\prodedgenok \inregandobs{i}{c} \right) (\inregandobs{\edgermind}{c}-1)^\numk \right)\\
      &= -\frac{\numwithobs{c}}{\totalhypoths}  \sumhe \prodedge \inregandobs{i}{c} + \frac{\numwithobs{c}-1}{\totalhypoths-1} \left( \sumhe \left(\prodedgenok \inregandobs{i}{c} \right) \left((\inregandobs{\edgermind}{c})^\numk - (\inregandobs{\edgermind}{c})^\numk + (\inregandobs{\edgermind}{c}-1)^\numk \right) \right)\\
      &= -\frac{\numwithobs{c}}{\totalhypoths}  \sumhe \prodedge \inregandobs{i}{c} + \frac{\numwithobs{c}-1}{\totalhypoths-1} \left( \sumhe  \prodedge \inregandobs{i}{c} - \sumhek \left( \prodedgenok \inregandobs{i}{c} \right) \left( (\inregandobs{\edgermind}{c})^\numk - (\inregandobs{\edgermind}{c}-1)^\numk \right) \right)\\
      &= -\frac{\totalhypoths - \numwithobs{c}}{\totalhypoths(\totalhypoths-1)}  \sumhe \prodedge \inregandobs{i}{c} - \frac{\numwithobs{c}-1}{\totalhypoths-1} \left( \sumhek \left( \prodedgenok  \inregandobs{i}{c} \right) \left( (\inregandobs{\edgermind}{c})^\numk - (\inregandobs{\edgermind}{c}-1)^\numk \right) \right)\\
      &\leq 0 \quad \text{(since each term $\leq 0$)}
    \end{align*}

    We also define:
    \begin{align*}
      \Delta^c &= \left(\frac{\totalhypoths - \numwithobs{c}}{\totalhypoths(\totalhypoths-1)}\right)  \Delta^c_1 +  \left(\frac{\numwithobs{c}-1}{\totalhypoths-1}\right)\Delta^c_2\\
      \Delta^c_1 &= -\sumhe \prodedge \inregandobs{i}{c}\\
      \Delta^c_2 &= - \left( \sumhek \left( \prodedgenok  \inregandobs{i}{c} \right) \left( (\inregandobs{\edgermind}{c})^\numk - (\inregandobs{\edgermind}{c}-1)^\numk \right) \right)\\
      \\
      \Delta^a &= \left(\frac{\totalhypoths(\totalhypoths - \numwithobs{c})}{\totalhypoths(\totalhypoths-1)}  + \frac{\numwithobs{c}-1}{\totalhypoths-1} \right) \Delta^a \\
      &= \left(\frac{\totalhypoths - \numwithobs{c}}{\totalhypoths(\totalhypoths-1)}\right) \Delta^a_1 +  \left(\frac{\numwithobs{c}-1}{\totalhypoths-1}\right)\Delta^a_2 \\
      \Delta^a_1 &= \totalhypoths \sumhek \left(\prodedgenok\inreg{i} \right) \left(\inregrm^\numk - (\inregrm-1)^\numk \right)\\
      \Delta^a_2 &=  \sumhek \left(\prodedgenok\inreg{i} \right) \left(\inregrm^\numk - (\inregrm-1)^\numk \right)
    \end{align*}

    The constants in front of the sum for $\Delta^c_1$ and $\Delta^c_2$ were from the equation, and $\Delta^a$ was split up to include the same constants. Now we will show that $\Delta^a_1 + \Delta^c_1 \geq 0$ and $\Delta^a_2 + \Delta^c_2 \geq 0$. We start with the latter:
    \begin{align*}
      \Delta^a_2 + \Delta^c_2 &= \sumhek \left[\left(\prodedgenok\inreg{i} \right) \left(\inregrm^\numk - (\inregrm-1)^\numk \right) -  \left( \prodedgenok  \inregandobs{i}{c} \right) \left( (\inregandobs{\edgermind}{c})^\numk - (\inregandobs{\edgermind}{c}-1)^\numk \right) \right] \\
      &\geq \sumhek \left[\left(\prodedgenok\inreg{i} \right) \left(\inregrm^\numk - (\inregrm-1)^\numk \right) - \left(\prodedgenok\inreg{i} \right) \left(\inregrm^\numk - (\inregrm-1)^\numk \right) \right] \numberthis \label{eq:dropobspart2}\\
      &= 0
    \end{align*}
    Where \eqref{eq:dropobspart2} follows from $\inreg{i} \geq \inregandobs{i}{c} \ \forall i$.

    \begin{align*}
      \Delta^a_1 + \Delta^c_1 &= \totalhypoths \sumhek \left(\prodedgenok\inreg{i} \right) \left(\inregrm^\numk - (\inregrm-1)^\numk \right) - \sumhe \prodedge \inregandobs{i}{c} \\
      &\geq \totalhypoths \sumhek \left(\prodedgenok\inreg{i} \right) \left(\inregrm^\numk - (\inregrm-1)^\numk \right) - \sumhe \prodedge \inreg{i} \numberthis \label{eq:dropobspart3} \\
      &\geq \totalhypoths \sumhek \left(\prodedgenok\inreg{i} \right) \inregrm^{\numk-1} - \sumhe \prodedge \inreg{i} \numberthis \label{eq:dropkorder} \\
      &= \left(\totalhypoths - \inregrm\right) \sumhek \left(\prodedgenok\inreg{i} \right) \inregrm^{\numk-1} + \sumhek \left(\prodedgenok\inreg{i} \right) \inregrm^{\numk} - \sumhe \prodedge \inreg{i}\\
      &= \left(\totalhypoths - \inregrm\right) \sumhek \left(\prodedgenok\inreg{i} \right) \inregrm^{\numk-1} - \sumhenok \prodedge \inreg{i} \numberthis \label{eq:canceledgek}\\
      &\geq \left(\totalhypoths - \inregrm\right) \sumhek \prodedgenok\inreg{i}  - \sumhenok \prodedge \inreg{i} \\
    \end{align*}
    
    Where \eqref{eq:dropobspart3} follows from $\inreg{i} \geq \inregandobs{i}{c} \ \forall i$, \eqref{eq:dropkorder} follows from $\inregrm^\numk - (\inregrm-1)^\numk \geq \inregrm^\numk - \inregrm^{\numk -1}(\inregrm-1) = \inregrm^{\numk -1}$, and \eqref{eq:canceledgek} cancels edges in $\hyperedgesetk$ exactly, leaving only edges in $\hyperedgesetnok$.

    We again want to separate out terms that cancel. We define:
    \begin{align*}
      \hyperedgeset^{\cardvar} &= \{\hyperedge:   \vert\hyperedge\vert=\cardvar \wedge\nexists\; {j} \text{ s.t. } \forall \subreg{ } \in\hyperedge: \subreg{ } \subseteq  \reg{j} \} \\
      \minhyperedgeset &= \{\hyperedge: \hyperedge \in \hyperedgeset, \nexists \widehat{\hyperedge} \subset \hyperedge : \widehat{\hyperedge} \in \hyperedgeset^{\maxcard - 1} \} \\
      \nominhyperedgeset &= \hyperedgeset \backslash \minhyperedgeset
    \end{align*}

    We defined $\hyperedgeset^{\cardvar}$ as the hyperedges for any specified cardinality $\cardvar$. We call $\minhyperedgeset$ the \emph{minimal} hyperedges if $\maxcard$ is the minimal cardinality at which these regions should be seperated. Thus, these are the hyperedges where no subset of subregions $\{\subreg{1} \dots \subreg{\maxcard-1}\} \subset \hyperedge$ would have a seperation hyperedge. All other hyperedges are called \emph{non-minimal}. We also define $\minhyperedgesetk, \minhyperedgesetnok, \nominhyperedgesetk, \nominhyperedgesetnok$ as the minimal and non-minimal hyperedges of $\hyperedgesetk$ and $\hyperedgesetnok$:
    \begin{align*}
      \minhyperedgesetk &= \{\hyperedge: \hyperedge \in \hyperedgesetk, \nexists \widehat{\hyperedge} \subset \hyperedge : \widehat{\hyperedge} \in \hyperedgeset^{\maxcard - 1} \} \\
      \minhyperedgesetnok &= \{\hyperedge: \hyperedge \in \hyperedgesetnok, \nexists \widehat{\hyperedge} \subset \hyperedge : \widehat{\hyperedge} \in \hyperedgeset^{\maxcard - 1} \} \\
      \nominhyperedgesetk &= \hyperedgesetk \backslash \minhyperedgesetk\\
      \nominhyperedgesetnok &= \hyperedgesetnok \backslash \minhyperedgesetnok\\
    \end{align*}
    
    We also note that:
  \begin{align*}
    \sum_{\hyperedge \in \nominhyperedgesetnok} \prodedge \inreg{i} &\leq \sum_{\subreg{j} \in \allsubreg \backslash \subregrm} \inreg{j} \sum_{\hyperedge \in \hyperedgesetnok^{\maxcard-1}} \prodedge \inreg{i} \\
    &= \left(\totalhypoths - \inregrm \right) \sum_{\hyperedge \in \hyperedgesetnok^{\maxcard-1}} \prodedge \inreg{i}
  \end{align*}

  For convenience, we define one additional set of hyperedges $\widehat{\hyperedgesetk}$. These are hyperedges in $\hyperedgesetk$ such that no subset of $\maxcard-1$ elements which do not include $k$ are in $\hyperedgesetnok$.
  \begin{align*}
    \widehat{\hyperedgesetk} &= \{ \hyperedge: \hyperedge \in \hyperedgesetk \; \wedge \; \nexists \hyperedge_{\maxcard-1} \subset \hyperedge \text{ s.t. } \hyperedge_{\maxcard-1} \in \hyperedgesetnok^{\maxcard-1} \}
  \end{align*}

  This enables us to split the set $\hyperedgesetk$ up into edges where $\hyperedgesetnok^{\maxcard-1}$ are a subset, and $\widehat{\hyperedgesetk}$. We note that since there is no region shared by all elements $\hyperedge^{\maxcard-1} \in \hyperedgesetnok^{\maxcard-1}$, then there will be no region shared by $\hyperedge = \hyperedge^{\maxcard-1} \cup \subregrm$. Thus, this will be an element of $\hyperedgesetk$. This gives us:
  \begin{align*}
    \sumhek \prodedgenok\inreg{i} &=  \sum_{\hyperedge \in \hyperedgesetnok^{\maxcard-1} } \prodedge \inreg{i} + \sum_{\hyperedge \in \widehat{\hyperedgesetk} } \prodedgenok \inreg{i}
  \end{align*}

  Applying these:
  \begin{align*}
      \Delta^a_1 + \Delta^c_1 &\geq \left(\totalhypoths - \inregrm\right) \sumhek \prodedgenok\inreg{i}  - \sumhenok \prodedge \inreg{i} \\
      &= \left(\totalhypoths - \inregrm\right) \left(  \sum_{\hyperedge \in \hyperedgesetnok^{\maxcard-1} } \prodedge \inreg{i} + \sum_{\hyperedge \in \widehat{\hyperedgesetk} } \prodedgenok \inreg{i} \right)  - \sumnominhenok \prodedge \inreg{i} - \summinhenok \prodedge \inreg{i} \\
      &\geq \left(\totalhypoths - \inregrm\right) \left(  \sum_{\hyperedge \in \hyperedgesetnok^{\maxcard-1} } \prodedge \inreg{i} + \sum_{\hyperedge \in \widehat{\hyperedgesetk} } \prodedgenok \inreg{i} \right)  - \left(\totalhypoths - \inregrm \right) \sum_{\hyperedge \in \hyperedgesetnok^{\maxcard-1}} \prodedge \inreg{i}- \summinhenok \prodedge \inreg{i} \\
      &= \left(\totalhypoths - \inregrm\right) \left( \sum_{\hyperedge \in \widehat{\hyperedgesetk} } \prodedgenok \inreg{i} \right)  -  \summinhenok \prodedge \inreg{i}
  \end{align*}

    At this point, we use the structure of our edge construction and definition of $\maxcard$ to show this sum is $\geq 0$. We have a positive term, consisting of edges which include $k$, and a negative term, consisting of edges that do not include $k$. We will show that for every product in the negative term, there is a corresponding product in the positive term. 

    To do so, we show that for any $\hyperedge \in \minhyperedgesetnok$, there is at least one corresponding $\hyperedge' \in \widehat{\hyperedgesetk}$ to cancel the terms out. More concretely:
    \begin{subsublemma}
      \label{subsublemma:minedgecancel}
      Let $\hyperedge \in \minhyperedgesetnok$. There exists some $\hyperedge^{\maxcard-1} \subset \hyperedge, \vert\hyperedge^{\maxcard-1}\vert = \maxcard-1$ such that $\hyperedge' = (\hyperedge^{\maxcard-1} \cup \subregrm) \in \widehat{\hyperedgesetk}$.
    \end{subsublemma}
    \begin{proof}
    Recall that $\hyperedge$ is a multiset of subregions. It is straightforward to see that because $\hyperedge$ is minimal, there can be no repeated elements in the multiset - and thus it is equivalent to a set. Define this set as $\hyperedge = \{\subreghat{1} \dots \subreghat{\maxcard} \}$. Define each distinct subset which does not include $\subreghat{i}$ as $\hyperedge_i = \hyperedge \backslash \subreghat{i}, 1 \leq i \leq \maxcard$. By our definition of minimal hyperedges $\minhyperedgesetnok$, we know that $\forall \hyperedge_{i}, \exists \reg{i} : \hyperedge_{i} \subseteq \reg{i}$, which implies that $\hyperedge_i \not\in \hyperedgesetnok^{\maxcard-1}$.  Note that each $\reg{i}$ must be distinct. If $\reg{i} = \reg{j}$, for any $i,j$, this would imply $(\hyperedge_i \cup \hyperedge_j) = \hyperedge \in \reg{i}$. But since there exists a separating hyperedge $\hyperedge, \not\exists \regnoarg : \hyperedge \subseteq \regnoarg$. This implies $\subreghat{i} \not\subseteq \reg{i}$. Combining this with our definition of $\widehat{\hyperedgesetk}$, if $\not\exists \regnoarg : (\hyperedge_{i} \cup \subregrm) \subseteq \regnoarg \Rightarrow (\hyperedge_i \cup \subregrm) \in \widehat{\hyperedgesetk}$. To prove this lemma, we will show that this region cannot exist for all $\hyperedge_{i}$.

    If $\subregrm \not\subseteq \reg{i} \Rightarrow \hyperedge_i \cup \subregrm \not\subseteq \reg{i}$. For the sake of contradiction, suppose $\subregrm \subseteq \reg{i} \forall i$. Let $\regset = \{\reg{1} \dots \reg{\maxcard}\}$. For this to be true, it must be that: $\circled{1} \forall \hypoth \in \subregrm, \hypoth \in \regset \quad \circled{2} \forall \reg{i} \in \regset, \forall \widehat{\hypoth} \in \subreghat{i} : \widehat{\hypoth} \notin \reg{i}, \widehat{\hypoth} \in \regset \backslash \reg{i}$ where $\vert\regset\vert = \maxcard$. However, by definition of $\maxcard$ this cannot be true: the largest such $\regset$ where this holds $\vert\regset\vert = \maxcard-1$. Thus, we have a contradiction, and have shown such a set of regions $\{\reg{1} \dots \reg{\maxcard}\} = \regset : \subregrm \subseteq \reg{i} \; \forall \reg{i}$ cannot exist. Therefore, $\exists \hyperedge_i : (\hyperedge_i \cup \subregrm) \in \widehat{\hyperedgesetk}$.
  \end{proof}

  In order to apply Lemma~\ref{subsublemma:minedgecancel}, we split every $\hyperedge \in \minhyperedgesetnok$ it up into $\hyperedge^{\maxcard-1}$ and $\overline{\subreg{ } }$, where $\hyperedge^{\maxcard-1}$ is the subset of $\hyperedge$ such that $(\hyperedge^{\maxcard-1} \cup \subregrm) \in \widehat{\hyperedgesetk}$, and $\overline{\subreg{ } } = \hyperedge \backslash \hyperedge^{\maxcard-1}$. Let $\overline{\inreg{ } }$ be the number of particles in subregion $\overline{\subreg{ } }$, which we will use in \eref{eq:splitforsublem}:
  \begin{align*}
      \Delta^a_1 + \Delta^c_1 &\geq \left(\totalhypoths - \inregrm\right) \left( \sum_{\hyperedge \in \widehat{\hyperedgesetk} } \prodedgenok \inreg{i} \right)  -  \summinhenok \prodedge \inreg{i} \\
      &=  \left(\totalhypoths - \inregrm\right) \left( \sum_{\hyperedge \in \widehat{\hyperedgesetk} } \prodedgenok \inreg{i} \right)  - \summinhenok \overline{\inreg{} } \prod_{i \in \hyperedge^{\maxcard-1}} \inreg{i \numberthis \label{eq:splitforsublem} }\\
      &\geq  \left(\totalhypoths - \inregrm\right) \left( \sum_{\hyperedge \in \widehat{\hyperedgesetk} } \prodedgenok \inreg{i} \right)  - \sum_{\subreg{j} \in \allsubreg \backslash \subregrm} \inreg{j} \left( \sum_{\hyperedge \in \widehat{\hyperedgesetk} } \prodedgenok \inreg{i} \right) \numberthis \label{eq:applysublem}\\
      &=  \left(\totalhypoths - \inregrm\right) \left( \sum_{\hyperedge \in \widehat{\hyperedgesetk} } \prodedgenok \inreg{i} \right)  - \left(\totalhypoths - \inregrm \right) \left( \sum_{\hyperedge \in \widehat{\hyperedgesetk} } \prodedgenok \inreg{i} \right) \\
      &= 0
  \end{align*}

  Where \eref{eq:applysublem} applies Lemma~\ref{subsublemma:minedgecancel}.

  At this point, we have shown that $\Delta = \Delta^a + \Delta^b + \Delta^c \geq 0$, since $\Delta^b \geq0$ and $\Delta^a + \Delta^c \geq 0$, which is what we needed to show.

  \end{proof}
  It is not hard to see that for any $\subtestobs \subseteq \widehat{\subtestobs} \subseteq \settestobs$, we could show that $\marginalfHEC(\actionitem\!\mid\!\subtestobs) \geq \marginalfHEC(\actionitem\!\mid\!\widehat{\subtestobs}_1) \geq \marginalfHEC(\actionitem\!\mid\!\widehat{\subtestobs}_2) \dots \geq \marginalfHEC(\actionitem\!\mid\!\widehat{\subtestobs})$ In other words, we can always find a sequence of removing one hypothesis at a time to get from $\subtestobs$ to $\widehat{\subtestobs}$ when $\subtestobs \subseteq \widehat{\subtestobs} \subseteq \settestobs$.

\end{proof}
\end{proof}

\subsection{\thmref{th:performance}: Greedy Performance Bound}
We would like to apply Theorem 5.8 of~\cite{golovin_adaptive_2011}. We have already shown adaptive submodularity and strong adaptive monotonicty in \sref{sec:asproof}. The theorem also requires that instances are \emph{self-certifying}, which means that when the policy knows it has obtained the maximum possible objective value immediately upon doing so. See~\cite{golovin_adaptive_2011} for details. As our objective is equivalent for all remaining hypotheses in $\version(\subtestobs)$, our function $\fhec$ is self-certifying.

The performance bound now follows directly from Theorem 5.8 of~\cite{golovin_adaptive_2011}. To apply the theorem, we needed to define two constants: a bound on the maximum value of $\fhec(\subtestobs)$, $Q=1$, and the minimum our objective function can change by, which corresponds to removing one hyperedge, $\eta=p_{\min}^{\maxcard}$. Plugging those into Theorem 5.8 of~\cite{golovin_adaptive_2011} gives $ \cost(\policyHEC) \leq (\maxcard\ln(1/p_{\min})+1) \cost(\policy^*)$.


\end{document}